\documentclass{article}

\usepackage[margin=1.25in]{geometry}
\usepackage{microtype}
\usepackage{graphicx}
\usepackage{booktabs} 
\usepackage{authblk}
\usepackage{amsmath,amsthm,amssymb}
\usepackage{subcaption}
\usepackage{enumitem}
\usepackage{algorithm}
\usepackage{hyperref}
\usepackage[noend]{algorithmic}
\usepackage{xcolor}
\usepackage{mathtools}
\usepackage{cuted}
\usepackage{changepage}

\newcommand{\algname}{ITLM }

\newcommand{\link}{\omega}
\newcommand{\estname}{TL }
\newcommand{\ratiotrue}{\alpha^\star}
\newcommand{\ratiolearn}{\alpha}

\newtheorem{theorem}{Theorem}

\newtheorem{lemma}[theorem]{Lemma}
\newtheorem{assumption}[theorem]{Assumption}

\newtheorem{definition}[theorem]{Definition}

\title{ Learning with Bad Training Data via Iterative Trimmed Loss Minimization }

\date{}
\author[1]{Yanyao Shen}
\author[1]{Sujay Sanghavi}
\affil[1]{Department of ECE,  The University of Texas at Austin}

\begin{document}

\maketitle

\begin{abstract}
In this paper, we study a simple and generic framework to tackle the problem of learning model parameters when a fraction of the training samples are corrupted. We first make a simple observation: in a variety of such settings, the evolution of training accuracy (as a function of training epochs) is different for clean and bad samples. 
Based on this we propose to iteratively minimize the {\em trimmed loss}, by alternating between (a) selecting  samples with lowest current loss, and (b)  retraining a model on only these samples. We prove that this process recovers the ground truth (with linear convergence rate) in generalized linear models with standard statistical assumptions. Experimentally, we demonstrate its effectiveness in three settings: (a) deep image classifiers with errors only in labels, (b) generative adversarial networks with bad training images, and (c) deep image classifiers with adversarial (image, label) pairs (i.e., backdoor attacks). 
For the well-studied setting of random label noise, our algorithm achieves  state-of-the-art performance without having access to any a-priori guaranteed clean samples. 
\end{abstract}

\section{Introduction}
\label{sec:intro}

State of the art accuracy in several machine learning problems now requires training very large models (i.e. with lots of parameters) using very large training data sets. Such an approach can be very sensitive to the quality of the training data used; this is especially so when the models themselves are expressive enough to fit all data (good and bad) in way that may generalize poorly if data is bad. We are interested both in {\bf poorly curated} datasets -- label errors\footnote{For example, a faulty CIFAR-10 dataset with 30\% of automobile images mis-labeled as ``airplane" (and so on for the other classes) leads to the accuracy of a neural architecture like WideResNet-16 of \cite{Zagoruyko2016WRN} to go from over 90\% to about 70\%.} in supervised settings, and irrelevant samples in unsupervised settings -- as well as  situations like {\bf backdoor attacks} \cite{gu2017badnets} where a small number of adversarially altered samples (i.e. labels and features changed) can compromise security. These are well-recognized issues, and indeed several approaches exist for each kind of training data error; we summarize these in the related work section. However, these approaches are quite different, and in practice selecting which one to apply would need us to know / suspect the form of training data errors a-priori. 

In this paper we provide a single, simple approach that can deal with several such tainted training data settings, based on a {\bf key observation}. We consider the (common) setup where training proceeds in epochs / stages, and inspect the {\em evolution} of the accuracy of the model on the training samples -- i.e. after each epoch, we take the model at that stage and see whether or not it makes an error for each of the training samples. Across several different settings with errors/corruptions in training data, we find that {\em the  accuracy on ``clean" samples is higher than on the ``bad" samples, especially in the initial epochs of training}. 
Figure \ref{fig:motivation} shows four different settings where this is the case. 
This observation suggests a natural approach: iteratively alternate between {\em (a)} filtering out samples with large (early) losses, and {\em (b)} re-training the model on the remaining samples. Both steps can be done in pretty much any machine learning setting: all that is needed is for  one to be able to evaluate losses on training samples, and re-train models from a new set of samples.

\begin{figure}
	\centering
	\includegraphics[width=\linewidth]{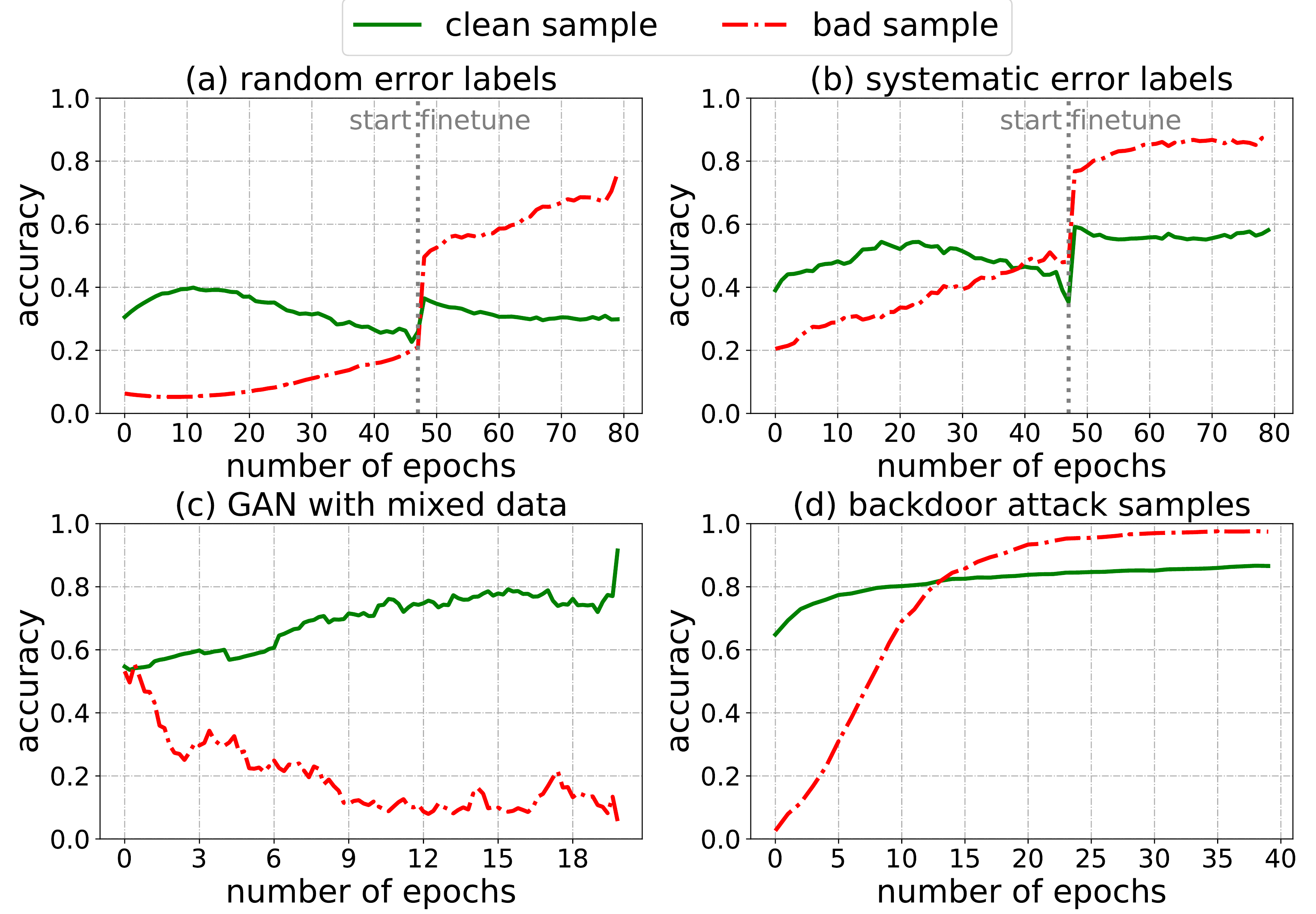}
	\caption{ {\bf Observation:} Evolution of model accuracy for clean and bad samples, as a function of training epochs, for four different tainted data settings: (a) classification for CIFAR-10 with  40\% random errors in labels, (b) classification for CIFAR-10 with 40\% systematic errors in labels, (c) DC-GAN trained on unlabeled mixture of 70\% MNIST images with 30\% Fashion-MNIST images, (d) backdoor attack on classification for CIFAR-10 with 250 watermarked backdoor attach samples as described in \cite{tran2018spectral}. The CIFAR-10 classifications are done using the WideResNet-16 of \cite{Zagoruyko2016WRN}. In all instances models are trained on the respective tainted data. {\bf Early on, models are more accurate on the good samples. }}
	\label{fig:motivation}
\end{figure}

Our approach is related to a classical statistical objective: minimizing the {\bf trimmed loss}. Specifically, given a set of $n$ samples, standard estimation / model-fitting involves choosing model parameters $\theta$ to minimize a loss function over all $n$ samples; in contrast, the trimmed loss estimator involves jointly choosing a subset of $\alpha n$ samples and $\theta$ such that the loss on the subset is minimum (over all choices of subset and parameters). This objective is intractable in general; our approach can be viewed as an iterative way to minimize trimmed loss. 
We describe trimmed loss\footnote{Our framework sounds initially similar to EM-style algorithms like $k$-means. Note however that EM needs to postulate a model for all the data points, while we search over a subset and do not worry about the loss on corrupted points. We are alternating between a simple search over subsets and a fitting problem on only the {\em selected} subset; this is not an instance of EM.}, its known properties, and some new results, in Section \ref{sec:estimator}. 
Our approach (in Section \ref{sec:itlm}) is thus an iterative way to find a trimmed loss estimator, hence we choose to call it Iterative Trimmed Loss Minimization (ITLM). 
 
We propose \algname as a generic approach to the problem of training with tainted data and 
investigate its performance both theoretically and empirically. More specifically, \textbf{our contributions} include: \\
{\bf (a)} In {\bf Section \ref{sec:itlm}}, we analyze \algname applied to a setting where the clean samples come from a ground-truth generalized linear model, and the bad samples have the response variables being (i) arbitrary corruption; (ii) random output; (iii) mixture output. We show \algname converges at least linearly to the ground truth model in all these settings. Our theoretic findings are further verified with synthetic experiments in {\bf Section \ref{sec:sub-synthetic}}.We also include a basic asymptotic property for general functions in {\bf Section \ref{sec:estimator}}; \\
{\bf (b)} In {\bf Section \ref{sec:classification}}, we show \algname can be applied to classification problems with bad labels. For CIFAR-10 classification with random labels, \algname performs better than previous state-of-the-art results, without using any identified clean sample;  \\
{\bf (c)} In {\bf Section \ref{sec:generation} and \ref{sec:backdoor}},  we succesfully apply \algname  to image generation task with bad images and classification task with adversarial (image, label) pairs (backdoor attacks). 

\paragraph{Notations} For integer $m$, $[m]$ denotes the set $\{0,\cdots, m-1\}$. For real number $a$, $\lfloor a\rfloor $ denotes the maximum integer no greater than $a$. $\sigma_{\min}$ and $\sigma_{\max}$ are the minimum/maximum eigenvalues. $a\wedge b$ and $a\vee b$ are shorthands for $\min\{a,b\}, \max\{a,b\}$. $|S|$ is the cardinality for set $S$. For two sets $S_1, S_2$, $S_1\backslash S_2$ is the set of elements in $S_1$ but not in $S_2$. The term \textit{w.h.p.} means \textit{with probability at least $1-n^{-c}$ where $c$ is an aribatrary constant}.

\section{Related Work}
There is a vast literature on bad training data problems. 
We classify the most related work from classic statistics to machine learning frontiers into the following four genres. 

\paragraph{Robust regression} 
There are several classes of robust estimators \cite{huber2011robust}. Among them, Least Trimmed Square (LTS) estimator~\cite{rousseeuw1984least} has high breakdown point and is sample efficient. Following the idea of LTS, several recent works provide algorithmic solutions and analyze their theoretical guarantees \cite{bhatia2015robust,vainsencher2017ignoring,yang2018general}. 
Different from previous works, we provide a fine characterization of the convergence in several settings, which connects to the problems of noisy labels/adversarial backdoor attack in practice. 
We also experimentally explore the overall approach for more complex tasks with deep neural network models. 
Notice that our approach is certainly not the only algorithm solution to finding least trimmed estimators. For example, see \cite{hossjer1995exact,rousseeuw2006computing,shen2013approximate} for algorithm solutions finding the least trimmed loss estimator in linear regression setting. However, compared to other works, our approach is more scalable, and not sensitive to the selection of loss functions.
Another line of recent work on robust regression consider strong robustness where the adversary poisons both the inputs and outputs, in both low-dimensional \cite{diakonikolas2018sever,prasad2018robust,klivans2018efficient} and high dimensional \cite{chen2013robust, balakrishnan2017computationally, liu2018high} settings. 
These algorithms usually require much more computation compared with, e.g., the algorithm we consider in this paper. 

\paragraph{Mixed linear regression}
Alternating minimization type algorithms are used for mixed linear regression with convergence guarantee~\cite{yi2014alternating, balakrishnan2017statistical}, in the setting of two mixtures. 
For  multiple mixture setting, techniques including tensor decomposition are  used~\cite{yi2016solving,zhong2016mixed,sedghi2016provable,li2018learning}, but require either high sample complexity or high computation complexity (especially when number of mixtures is large). 
On the other hand, ~\cite{ray2018search} studies finding a single component in mixture problems using a particular type of side information. 

\paragraph{Noisy label problems}
Classification tasks with noisy labels are also  of wide interest.  \cite{frenay2014comprehensive} gives an overview of the related methods. Theoretical guarantee for noisy binary classification has been studied under different settings ~ \cite{scott2013classification,natarajan2013learning,menon2016learning}. 
More recently, noisy label problem has been studied for DNNs.  
\cite{reed2014training} and \cite{malach2017decoupling} develop the idea of bootstrapping and query-by-committee for DNNs. 
On the other hand, 
\cite{khetan2017learning} and \cite{zhang2018generalized} provide new losses for training under the noise.  \cite{sukhbaatar2014learning} adds a noise layer into the training process, while 
\cite{ren2018learning} provides a meta-algorithm for learning the weights of all samples by heavily referencing to a clean validation data during training. 
\cite{jiang2017mentornet} proposes a data-driven curriculum learning approach. 

\paragraph{Defending backdoor attack} 
Several recent works defend against backdoor attack samples~\cite{gu2017badnets}. 
\cite{tran2018spectral} proposes using spectral signature, where they calculate the top singular vector of a certain layer's representation for all the samples. 
\cite{liu2018fine} proposes pruning  DNNs based on the belief that backdoor samples exploit spare capacity.  \cite{wangneural} uses a reverse engineering approach. \cite{chen2018detecting} detect by activation clustering. 
While the adversary is not allowed to train the model, these approaches do not exploit the evolution of the training accuracy for detecting backdoor samples.

\section{Setup and (Exact) Trimmed Loss Estimator}
\label{sec:estimator}

We now describe the least trimmed loss estimator in general. Let  $s_1,\ldots, s_n$ be the samples, $\theta$ the  model parameters to be learnt, and loss function $f_\theta(\cdot)$. With this setting, the standard approach is to minimize the total loss of all samples, i.e.  $\min_\theta \sum_i f_\theta(s_i)$. In contrast, the least trimmed loss estimator is given by
\begin{align*}
\hat{\theta}^{(\mathtt{TL})} =  \arg\min_{\theta \in \mathcal{B}} \min_{S :|S| = \lfloor \alpha n\rfloor } \sum_{i\in S} {f}_{\theta}(s_i). 
\end{align*}
For finding $\hat{\theta}^{(\mathtt{TL})}$ we need to minimize over {\em both} the set $S$ of size $\lfloor \alpha n \rfloor $ -- where $\alpha\in(0,1)$ is the fraction of samples we want to fit -- and the set of parameters $\theta$. In general solving for the least trimmed loss estimator is hard, even in the linear regression setting~\cite{mount2014least}, i.e., even when $s = (x,y)$ and $f_{\theta}(x,y) = (y-\theta^\top x)^2$. Nevertheless, its statistical efficiency has been studied. In the linear setting, it has a breakdown point of $1/2$ asymptotically~\cite{huber2011robust}, and  is consistent~\cite{vivsek2006least}, i.e., $\hat{\theta}^{(\mathtt{\estname})} \rightarrow \theta^\star$ in probability as $n\rightarrow \infty$. \cite{vcivzek2008general} also shows this property for more general function classes. 

We now present a basic result for more general non-linear functions. Let $\mathcal{B}$ be a compact parametric space, and all samples are i.i.d. generated following certain distribution. 
For $\theta \in \mathcal{B}$, let $D_{\theta}, d_{\theta}$ be the distribution and density function of $f_{\theta}(s)$. 
Let $S(\theta) = \mathbb{E}_s[f_{\theta}(s)]$ be the population loss, and let $S_n(\theta) = \frac{1}{n}\sum_{i=1}^{n} f_{\theta}(s_i)$ be the empirical loss. 
Define $F(\theta) := \mathbb{E}\left[ f_{\theta}(s) \mathbf{I}(f_{\theta}(s)\le D_{\theta}^{-1}(\alpha)) \right]$ as the population trimmed loss. 
Let $\mathcal{U}(\theta, \epsilon) := \{\tilde{\theta} \mid |S(\tilde{\theta}) - S(\theta) | < \epsilon, \tilde{\theta} \in \mathcal{B} \}$  be the set of parameters with population loss close to $\theta$. 
We require the following two natural assumptions: 

\begin{assumption}[Identification condition for $\theta^\star$] \label{ass:ide}
For every $\epsilon>0$ there exists a $\delta>0$ such that if $\theta \in \mathcal{B}\backslash \mathcal{U}(\theta^\star, \epsilon)$,  
we have that $F(\theta) - F(\theta^\star)  > \delta$.
\end{assumption}

\begin{assumption}[Regularity conditions] \label{ass:reg}
	$D_{\theta}$ is absolutely continuous for any $\theta \in \mathcal{B}$. 
	$d_{\theta}$ is bounded uniformly in $\theta \in \mathcal{B}$, and is locally positive in a neighborhood of its $\alpha$-quantile. 
	$f_{\theta}(s)$ is differentiable in $\theta$ for  $\theta \in \mathcal{U}(\theta^\star, \epsilon)$,  for some $\epsilon >0$.
\end{assumption}

The identification condition identifies $\theta^\star$ as achieving the global minimum on the population trimmed loss. The regularity conditions are standard and very general. Based on these two assumptions, we show that \estname is consistent with $\theta^\star$ in empirical loss. 

\begin{lemma}
	Under Assumptions \ref{ass:ide} and \ref{ass:reg}, the estimator $\hat{\theta}^{(\mathtt{\estname})}$ satisfies:
	$
	\left\vert S_n(\hat{\theta}^{(\mathtt{\estname})}) - S_n(\theta^\star)\right\vert \rightarrow 0
	$
	with probability $1$, as $n\rightarrow \infty$.
\end{lemma}

\section{Iterative Trimmed Loss Minimization} \label{sec:algo}

Our approach to (attempt to) minimize the trimmed loss, by alternating between minimizing over $S$ and $\theta$ is described below.

\begin{algorithm}
	\centering
	\caption{Iterative Trimmed Loss Minimization ($\mathtt{\algname}$)} 
	\label{alg:1}
	\begin{algorithmic}[1]
		\STATE \textbf{Input:}  Samples $\{s_i\}_{i=1}^n$, number of rounds $T$,  fraction of samples $\alpha$
		\STATE \textbf{(Optional) Initialize:}  
		$\theta_0 \leftarrow  \arg\min_{\theta} \sum_{i\in [n]} f_{\theta}(s_i) $
		\STATE \textbf{For} {$t=0, \cdots, T-1$} \textbf{do}
		\STATE  \quad \quad Choose samples with smallest current loss $f_{\theta_t}$:
		\[ S_t \leftarrow \arg\min_{S : |S|=\lfloor \alpha n\rfloor } \sum_{i \in S} {f}_{\theta_t}(s_i)\] \label{alg:step-4} 
		\STATE \quad \quad $\theta_{t+1} = \mathtt{ModelUpdate}(\theta_t, S_t, t)$
		
		\STATE \textbf{Return:} $\theta_T$
	\end{algorithmic}
\end{algorithm}

Here $\mathtt{ModelUpdate}(\theta_t, S_t, t)$ refers to the process of finding a new $\theta$ given sample set $S_t$, using $\theta_t$ as the initial value (if needed) in the update algorithm, and also the round number $t$. For example this could just be the (original, naive) estimator that minimizes the loss over all the samples given to it, which is now $S_t$. 

In this paper we will use batch stochastic gradient as our model update procedure, so we now describe this for completeness.

\begin{algorithm}
	\centering
	\caption{$\mathtt{BatchSGD\_ModelUpdate}(\theta, S, t)$ } 
	\label{alg:2}
	\begin{algorithmic}[1]
		\STATE \textbf{Input:} Initial parameter $\theta$, set $S$, round $t$
		\STATE \textbf{Choose: } Step size $\eta$, number $M$ of gradient steps, batch size $N$
		\STATE \textbf{(Optional)} Re-initialize $\theta^0$ randomly  
		\STATE  \textbf{For} {$j=1,\cdots, M$} \textbf{do}
		\STATE  \quad\quad  $B_j \leftarrow \mathtt{random\_subset}(S, N)$
		\STATE  \quad\quad $\theta^j \leftarrow \theta^{j-1} - \eta \left ( \frac{1}{N} \sum_{i\in B_j}  \nabla_\theta f_{\theta^{j-1}} (s_i) \right )$
		\STATE \textbf{Return: } $\theta^M$
	\end{algorithmic}
\end{algorithm}

Note that for different settings, we use the same procedure as described in Algorithm \ref{alg:1} and \ref{alg:2}, but may select different hyper-parameters. We will clarify the alternatives we use in each part.  

\section{Theoretical Guarantees for Generalized Linear Models} \label{sec:itlm}

We now analyze \algname for generalized linear models with errors in the outputs (but not in the features): we are given samples each of the form $(x,y)$ such that
\begin{equation}
\begin{split}
	y = &  \link( \phi(x)^\top \cdot \theta^\star )  + e, \quad \mbox{ (clean samples)}   \\
	y = &  r + e, \quad\quad \quad \quad \,\,\, \quad \quad \mbox{ (bad samples)} 
\end{split}
\label{eqt:formulation}
\end{equation}
Here $x$ represents the inputs, $y$ the output, embedding function $\phi$ and link function $w$ are known (and possibly non-linear) \footnote{In neural network models, for example, $\phi(x)$ would represent the output of the final representation layer, and we assume that the parameters of the previous layers are fixed.}, $e$ is random subgaussian noise with parameter $\sigma^2$~\cite{vershynin2010introduction}, and $\theta^\star$ is the ground truth. Thus there are errors in outputs $y$ of bad samples, but not the features. Let  $\alpha^\star$ be the fraction of clean samples in the dataset. 

For \algname, we use squared loss, i.e. $f_\theta(x,y) = (y -  \link( \phi(x)^\top \cdot \theta) )^2$. We will also assume the feature matrices are regular, which is defined below.
\begin{definition}[]
	Let $\Phi(X) \in \mathbb{R}^{n\times d}$ be the feature matrix for all  samples, where the $i$th row is  $\phi(x_i)^\top$. Let $\mathcal{W}_k=\{ W\in \mathbb{R}^{n\times n} \vert W_{i,j} =0, W_{i,i}\in \{0,1\}, \mathtt{Tr}(W) = k  \}$.  Define
	\begin{align*}
		\psi^{-}(k) = & \min_{W: W\in \mathcal{W}_k} \sigma_{\min} \left( \Phi(X)^\top W\Phi(X) \right),  \\
		\psi^{+}(k) = & \max_{W:  W\in \mathcal{W}_k} \sigma_{\max} \left(  \Phi(X)^\top W \Phi(X) \right).
	\end{align*}
	We say that $\Phi(X)$ is a regular feature matrix if for $k=\alpha n$, $\alpha \in [c,1]$, $\psi^{-}(k) = \psi^{+}(k) = {\Theta}(n)$ for  $n = {\Omega}(d\log d)$. 
\end{definition}
Regularity states that every large enough subset of samples results in a $\Phi(X)$ that is well conditioned. This holds under several natural settings, please see Appendix for more discussion. We now first present a one-step update lemma for the linear case.

\begin{lemma}[{\bf linear case}] \label{thm:alg}
	Assume $\link (x) = x$ and we are using \algname with $\alpha$. The (for large enough $M$ and small  $\eta$ in Algorithm \ref{alg:2}), the following holds per round update w.h.p.: 
	\begin{align*}
		&\| \theta_{t+1} - \theta^\star \|_2 \le \frac{\sqrt{2}\gamma_t}{\psi^{-}(\alpha n)}  \|\theta_{t} - \theta^\star\|_2  + \frac{
			\sqrt{2} \varphi_t + c \xi_t  \sigma }{\psi^{-}(\alpha n)},
	\end{align*}
	where	$
	\varphi_t =\left\| \sum_{i\in S_t\backslash S^\star} (\phi(x_i)^\top \theta_t - r_i - e_i) \phi(x_i) \right\|_2,$
	and
	$\gamma_t = \psi^{+}(|S_t\backslash S^\star|),  \,\,\xi_t= {\sqrt{\sum_{i=1}^n \|\phi(x_i)\|_2^2 \log n  } }.
	$
\end{lemma}
This Lemma \ref{thm:alg} bounds the error in the next step based on the error in the current step, how mismatched the set $S_t$ is as compared to the true good set $S^\star$, and the regularity parameters. The following does the same for the more general non-linear case. 

\begin{lemma}[{\bf non-linear case}]\label{thm:nonlinear}
	Assume $\link: \mathbb{R}\rightarrow \mathbb{R}$  monotone and differentiable. Assume $\link'(u)\in [a,b]$ for all $u\in \mathbb{R}$, where $a,b$ are positive constants. Then, for \algname with $\alpha$ (and $M=1$ and $N=|S|$ in Algorithm \ref{alg:2}),  w.h.p., 
	\begin{align*}
	&\|\theta_{t+1} - \theta^\star \|_2 
	\le \left( 1\! -\! \frac{\eta}{\alpha n}\!{a^2 \psi^{-}(\alpha n)} \! \right) \|\theta_t\! -\! \theta^\star \|_2 
	\!+\! {\eta}\frac{  \tilde{\varphi}_t   + \xi_t b \sigma }{\alpha n}, 
	\end{align*}
	where $\xi_t$ is the same as in  Lemma  \ref{thm:alg}, and
	$\tilde{\varphi} =  \left\| \sum_{i\in S_t\backslash S^\star} \left( w(\phi(x_i)^\top \theta^\star) - r_i-e_i\right) w'(\phi(x_i)^\top \theta^\star) \phi(x_i)\right\|$.
\end{lemma}

\paragraph{Remarks:} A few comments \\
{\bf (1)} Lemma \ref{thm:alg} and Lemma \ref{thm:nonlinear} directly lead to the consistency of the algorithm in the clean data setting ($\alpha^\star=1$). This is because $|S_t\backslash S^\star| = 0$, which makes both $\alpha_t$ and $\psi_t$ become zero. Moreover, $\xi_t$ is sublinear in $n$, and can be treated as the statistical error. 
While this consistency property is not very surprising (remind that Section \ref{sec:estimator} shows \estname estimator has consistent performance for very general class of functions), the update property helps us better analyze convergence behavior in multiple corruption settings. \\
{\bf (2)} In \cite{bhatia2015robust}, convergence of  the parameter is characterized by the linear convergence of $\sum_{i\in S_t}r_i^2$ with constant rate. 
Here, by directly characterizing the convergence in parameter space, we gain several additional benefits: 
(a) generality: we can directly analyze our result under several settings, including arbitrary corruption, random output, and mixture output; 
(b) finer chracterization: we see that the rate depends on $\alpha_t/ \psi^{-}(\alpha n)$, which goes down as long as $|S_t\backslash S^\star|$ goes down. We can have a finer characterization of the convergence, e.g., super-linear convergence for the latter two corruption settings. 

Next, we specialize our analysis into several bad training data settings. In the main paper we state Theorem \ref{thm:random} and Theorem \ref{theorem:se} for the linear setting $\link(x) = x$, while Theorem \ref{thm:app-1} and Theorem \ref{thm:app-2} in the appendix represent the same for the non-linear setting. 

Our first result shows per-round convergence when a fraction $1-\alpha^\star$ samples are either arbitrarily or randomly corrupted, and we choose the $\alpha$ in \algname to be $\alpha < \alpha^*$.
\begin{theorem}[{\bf arbitrary/random corruptions}] \label{thm:random} 
	Assume $\link (x) = x$. We are given  clean sample ratio $\alpha^\star > c_\mathtt{th}$, and \algname with $\alpha$ such that  $\alpha<  \alpha^\star$ and sample size $n=\Omega(d\log d)$. Then w.h.p., we have:
	\begin{align*}
		&\|\theta^\star\! -\! \theta_{t+1} \|_2 \le \kappa_t \|\theta^\star\! -\! \theta_t \|_2 + c_1\sqrt{ \kappa_t   } \sigma\!+\! \frac{c_2\xi_t  }{n} \sigma,
	\end{align*}
	where $\kappa_t \le \frac{1}{2}$ when $r$ is arbitrary, and  $\kappa_t \le c \{\sqrt{\|\theta_t - \theta^\star\|_2^2 + \sigma^2} \vee \frac{\log n}{n}\}$ when $r$ is random sub-Gaussian output. All the $c$ constants depend on the regularity conditions.
\end{theorem}

\paragraph{Remark} In both settings, we show at least  linear convergence performance for per-round update. 
On the other hand, even in the infinite sample setting, the second term would not go to zero, which implies that our theoretic guarantee does not ensure consistency. 
In fact, we show in the next section that our analysis is tight, i.e., \algname indeed gives inconsistent estimation. 
However, if the noise is small, \algname will converge very closely to the ground truth. 
The proof is in Appendix. 

We now consider the case when the data comes from a mixture model. We provide local convergence result that characterizes the performance of \algname for this setting. 
More specifically, the full set of samples $S = [n]$ is splitted into $m$ sets: $S = \cup_{j\in [m]}S_{(j)}$, each corresponding to samples from one mixture component, $|S_{(j)}| = \alpha_{(j)}^\star n$. 
The response variable $y_i$ is given by: 
\begin{align}\label{eqt:mlr}
y_i =  
\link \left( \phi(x_i)^\top \theta_{(j)}^\star \right)  + e_i, \mbox{ for } i\in S_{(j)}.
\end{align}
Fitted into the original framework, $r_i = \link(\phi(x_i)^\top \theta_{(j)}^{\star}) $ for some $j\in [m]\backslash \{0\}$. 
Similar to previous literatures \cite{yi2014alternating,zhong2016mixed}, we consider $\phi(x) \sim \mathcal{N}(0, I_d)$. 
We have the following  convergence guarantee in the mixture model setting: 

\begin{theorem}[{\bf mixed regression}] \label{theorem:se}
	Assume $\link (x) = x$ and consider \algname with $\alpha$. For the mixed regression setting in (\ref{eqt:mlr}), suppose that for some component $j\in [m]$, we have that $\alpha < \alpha_{(j)}^\star$. 
	Then, for  $n= \Omega( d\log d)$, w.h.p., the next iterate $\theta_{t+1}$ of the algorithm satisfies 
	\begin{align*}
	&\|\theta_{t+1}\! -\! \theta_{(j)}^\star \|_2 \le \kappa_t \|\theta_t\! -\! \theta_{(j)}^\star \|_2\! +\! c_1{\sqrt{ {\kappa_t }  }} \sigma\! +\! \frac{c_2\xi_t}{ n}\sigma, 
	\end{align*}
	where $\kappa_t \le c  \left\{  \frac{\sqrt{\|\theta_t - \theta_{(j)}^\star \|_2^2 + \sigma^2}}{\min_{k\in[m]\backslash \{j\}}  \sqrt{\|\theta_t - \theta_{(k)}^\star\|_2^2 + \sigma^2} } \vee \frac{\log n}{n} \right\}$. 
\end{theorem}

\paragraph{Remark} 
Theorem \ref{theorem:se} has a nearly linear dependence (sample complexity) on  dimension $d$. 
In order to let $\kappa_0 <1$,  the iterate $\theta_0$ in Algorithm \ref{alg:1} needs to satisfy $\|\theta_0 - \theta_{(j)}^\star \|_2\le C(\alpha) \min_{k\in [m]\backslash\{j\} }\|\theta_0 - \theta_{(k)}^\star \|_2 -  \sqrt{1- C(\alpha)^2} \sigma$, where $C(\alpha) = \min\{\frac{c_3\alpha}{1-\alpha}, 1\}$. 
For $\alpha$ large enough such that $C(\alpha)=1$, the condition on $\theta_0$ does not depend on $\sigma$. 
However,  for smaller $\alpha$, the condition of $\theta_0$ tolerates smaller noise.  
This is because, even if $\theta_0$ is very close to $\theta_{(j)}^\star$, when the noise  and the density of  samples from other mixture components are both high, the number of  samples from other components selected by the current $\theta_0$ would still be quite large, and the update will not converge to $\theta^\star_{(j)}$.

\section{Experiments}
\label{sec:exp}

\begin{figure*}[ht]
	
	\begin{subfigure}{0.32\columnwidth}
		\centering
		\includegraphics[width=\linewidth]{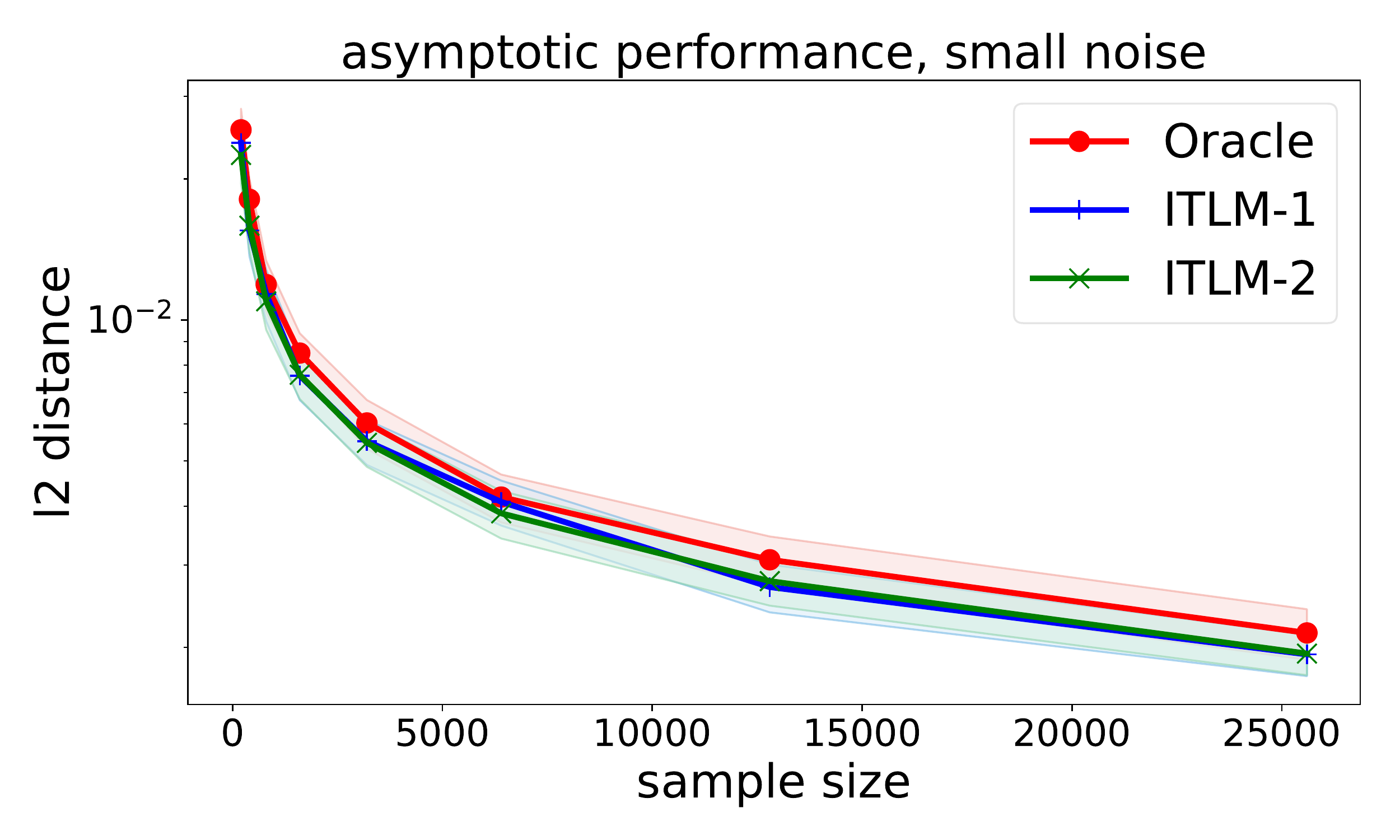}
		\caption*{(a)}
	\end{subfigure}
	\hfill 
	\begin{subfigure}{0.32\columnwidth}
		\centering
		\includegraphics[width=\linewidth]{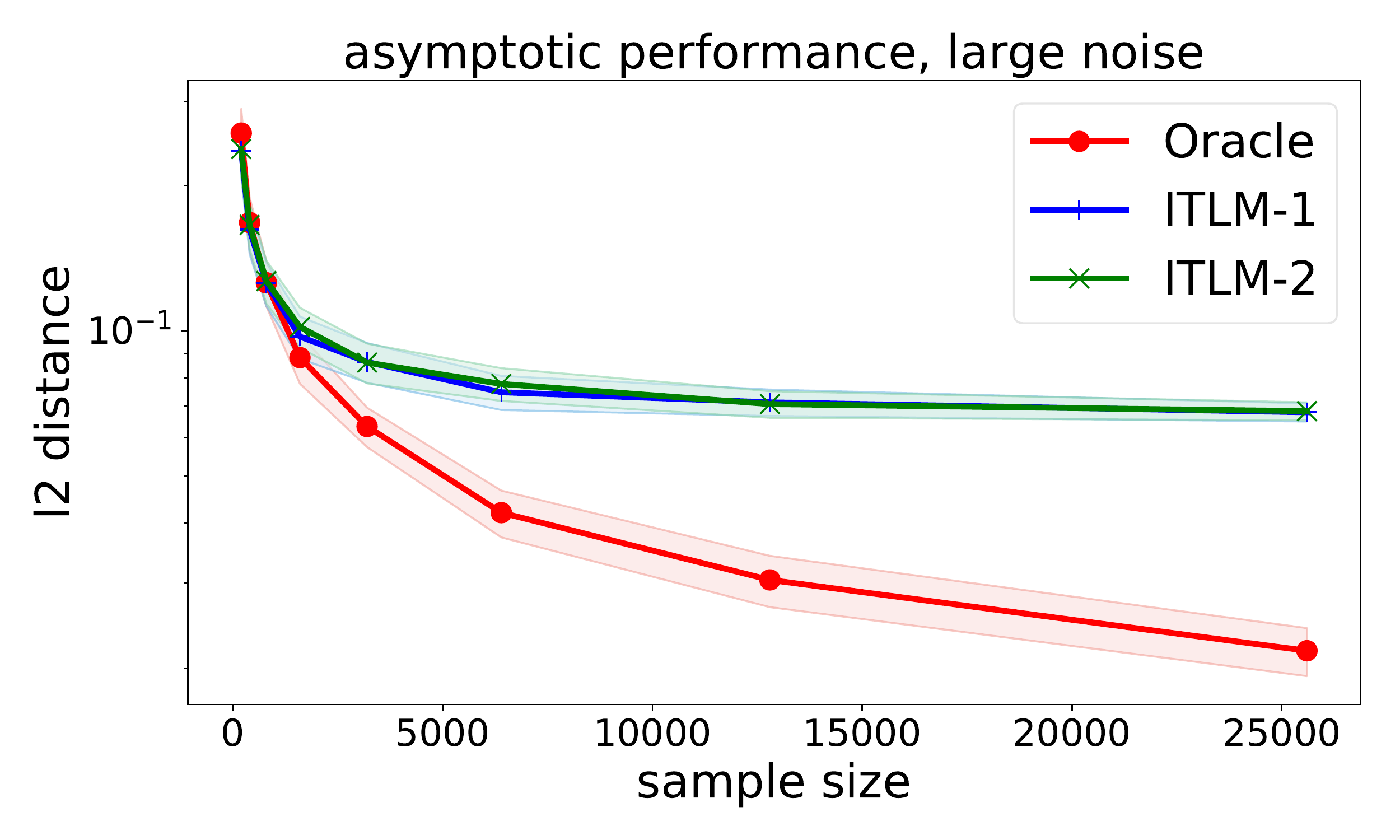}
		\caption*{(b)}
	\end{subfigure}
	\hfill 
	\begin{subfigure}{0.32\columnwidth}
		\centering
		\includegraphics[width=\linewidth]{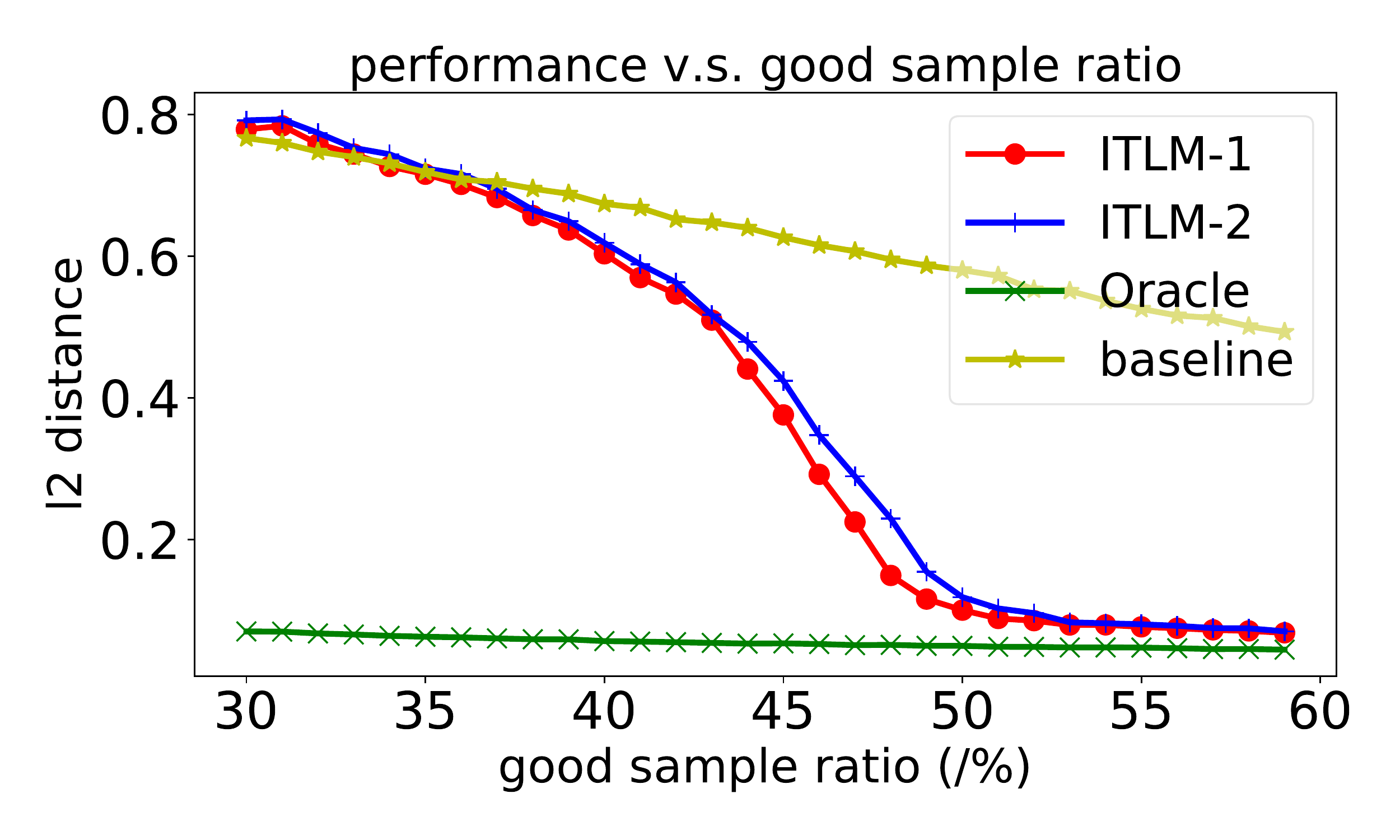}
		\caption*{(c)}
	\end{subfigure}

	\caption{Synthetic experiments: \textbf{(a):} $\|\theta^\star - \theta_T\|_2$ v.s. sample size under small measurement noise; \textbf{(b):} under large measurement noise; \textbf{(c):} $\|\theta^\star - \theta_T\|_2$ v.s. different good sample ratio. \algname-1: \algname with large $M$ (full update per round); \algname-2: \algname with $M=1$. 
	} 
	\label{fig:sim}
\end{figure*}

We first run synthetic experiments to verify and illustrate the theoretical guarantees we have in Section \ref{sec:itlm}. 
Then, we present the result of our algorithm in several bad training data settings using DNNs, and compare them with the state-of-the-art results. 
Although deep networks have the capacity to fit bad samples as well as the clean samples, 
as we motivated in Section \ref{sec:intro}, the learning curve for clean samples is better than that of the bad samples (at least in early stages of training), which aligns with the linear setting.
Besides, for the  noisy label problem,  we show that \algname performs  better than previous state-of-the-art methods where additional DNNs and additional clean samples are required. Our algorithm is simple to implement and requires neither of them. Training details are explained in  Appendix.

\subsection{Synthetic experiments}
\label{sec:sub-synthetic}
We consider the linear regression setting, where  dimension $d=100$ and sample $n=1000$, all $\phi(x_i)$s are generated as i.i.d. normal Gaussian vectors. 
The outputs are generated following (\ref{eqt:formulation}). 
The results are based on an average of $100$ runs under each setting. 
The performance for both the random output setting and the mixture model setting are similar, we focus on random output setting in this section. 
For similar results in mixture model setting, and further results in the general linear setting, please refer to the Appendix.

\textbf{Results}: 
\textit{(a) Inconsistency. } Figure \ref{fig:sim}-(a) and (b) show that  \algname gives an inconsistent result, since under the  large noise setting, the recovery error does not decrease as sample size increases.  
Also, (a) suggests that if the noise is small, the final performance is close to the oracle unless sample size is extremely large. 
These observations match with our theoretic guarantees in Theorem \ref{thm:random} and \ref{theorem:se}. 
\textit{(b) Performance v.s. ${\alpha}^\star$. } Figure \ref{fig:sim}-(c) shows the recovery result  of our algorithm, for both large and small $M_t$. As $\alpha^\star$ increases, \algname is able to successfully learn the parameter with high probability. 
\textit{(c) Convergence rates.} In fact, as implied by Theorem \ref{thm:random}, the algorithm has super-linear convergence under the small noise setting. We provide results in Appendix to verify this. 

Next, we apply \algname to many bad data settings with DNN models. We present the result using \algname with large $M$ since (a) according to Figure \ref{fig:sim}, both large and small $M$ perform similar in linear setting; (b) full update could be more stable in DNNs, since a set of bad training samples may deceive one gradient update, but is harder to deceive the full training process. Also, we run re-initialization for every round of update to make it harder to stuck at bad local minimum.

\subsection{Random/Systematic label error for classification}
\label{sec:classification}

\begin{table*}[t]
	\centering
	\caption{{\bf Neural networks classification accuracy with random/systematic label error:} Performance for subsampled-MNIST, CIFAR-10,  datasets as the ratio of clean samples varies. \textbf{Baseline} : Naive training using all the samples; \textbf{\algname}: Our iterative update algorithm with $\alpha = \alpha^\star - 5\%$; \textbf{Oracle} : Training with all clean samples. \textbf{Centroid}: Filter out samples far away from the centroid for each label class; \textbf{1-step}: The first iteration of \textbf{\algname}; $\Delta \alpha:10\%(15\%)$: \textbf{\algname} with $\alpha= \alpha^\star-10\%(15\%)$. 
	We see significant improvement of \textbf{\algname} over \textbf{Baseline} for all the settings. }
	\scalebox{0.75}{
	\begin{tabular}{lccccccccccccccc}
		\toprule 
		dataset & \multicolumn{7}{c}{MNIST with two-layer CNN} & \multicolumn{3}{c}{CIFAR-10 with WideResNet16-10} \\
		\midrule
		\multicolumn{11}{l}{\textbf{Systematic Label Error }} \\
		$\frac{\mbox{\#clean}}{\mbox{\#total}}$  & \textbf{Baseline} & \textbf{\algname} & \textbf{Oracle} & \textbf{Centroid} & \textbf{1-step} & $\Delta \alpha: 10\%$ & $\Delta \alpha: 15\%$& \textbf{Baseline} & \textbf{\algname} & \textbf{Oracle} & \\
		
		\midrule
		60\% & 66.69 & 84.98  & 92.44 &70.25&74.29&85.91&79.80& 62.03 & 81.01   &  90.14  \\
		70\% & 80.74&89.19&92.82&83.42&84.07&89.76&88.00& 73.47 & 87.08 &  90.72  \\
		80\%  & 89.91&91.93&92.93&90.18&91.38&90.92&89.06&80.17 &89.34 & 91.33  \\
		90\% & 92.35&92.68&93.2&92.44&92.63&91.10&90.62& 86.63 & 90.00 & 91.74  \\
		\midrule
		\multicolumn{11}{l}{\textbf{Random Label Error }} \\
		$\frac{\mbox{\#clean}}{\mbox{\#total}}$  & \textbf{Baseline} & \textbf{\algname} & \textbf{Oracle} & \textbf{Centroid} & \textbf{1-step} & $\Delta \alpha: 10\%$ & $\Delta \alpha: 15\%$& \textbf{Baseline} & \textbf{\algname} & \textbf{Oracle} & \\
		30\% & 80.87& 84.54& 91.37& 80.89& 93.91& 80.39& 68.00& 49.58 & 64.74 & 85.78 \\
		50\%  &88.59& 90.16&92.14&88.94&89.13&89.14&86.23 & 64.74 & 82.51 & 89.26 \\
		70\%  &91.18&91.12&92.82&91.25&90.28&90.41&88.37 & 73.60 & 88.23 &  90.72 \\
		90\% & 92.50&92.43&93.20&92.40&92.42&91.48&90.25 & 86.13 & 90.33  & 91.74  \\
		\bottomrule
	\end{tabular} }
	\label{tab:image-classification-se}
\end{table*}

We demonstrate the effectiveness of \algname for correcting training label errors in classification by starting from a ``clean" dataset, and introducing either one of two different types of errors to make our training and validation data set:
\begin{adjustwidth}{5pt}{}
	(a) {\em random errors in labels}: for samples in error, the label is changed to a random incorrect one, independently and with equal probability;  \\
	(b) {\em systematic errors in labels}: for a class ``a", all its samples in error are given the {\em same} wrong label ``b"
\end{adjustwidth}
Intuitively, systematic errors are less benign than random ones since the classifier is given a more directed and hence stronger misleading signal. However, systematic errors can happen when some pairs of classes are more confusable than others. 
We investigate the ability of \algname to account for these errors for
{\em (a)} 5\% subsampled MNIST~\cite{lecun1998gradient}~\footnote{We subsample MNIST by retaining only 5\% of its samples, so as to better distinguish the performance of different algorithms. Without this, the MNIST dataset is ``too easy" in the sense that it has so many samples that the differences in algorithm performance is muted if all samples are used.} dataset with a standard 2-layer CNN, 
and 
{\em (b)}  CIFAR-10~\cite{krizhevsky2009learning} with a 16-layer WideResNet
~\cite{Zagoruyko2016WRN}.
For each of these, \textbf{Baseline} represents the standard process of training the (respective) NN model on all the samples. 
Training details for each dataset are specified in the Appendix. 
For the CIFAR-10 experiments, we run $4$ rounds with early stopping, and then $4$ rounds with full training. 
As motivated in Section 1 (in the main paper), early stopping may help us better filter out bad samples since the later rounds may overfit on them. 
We set $\alpha$ to be $5\%$ less than the true ratio of clean samples, to simulate the robustness of our method to mis-specified sample ratio. 
Notice that one can always use cross-validation to find the best $\alpha$. 
For the MNIST experiments, we run $5$ rounds of \algname, and we also include comparisons to several heuristic baseline methods, to give more detailed comparison, see the caption thereof. 

\textbf{Results:} 
According to  Table  \ref{tab:image-classification-se}, we observe significant improvement over the baselines under most settings for MNIST experiments and all settings for CIFAR-10 experiments. 
\algname is also not very sensitive to mis-specified clean sample ratio (especially for cleaner dataset). 
We next compare  with two recent state-of-the-art methods  focusing on the noisy label problem: (1) MentorNet PD/DD~\cite{jiang2017mentornet} and (2) Reweight ~\cite{ren2018learning}. 
These methods are based on the idea of curriculum learning and meta-learning. 
MentorNet DD and Reweight require an additional clean dataset to learn from.  
As shown in Table \ref{tab:sota} (MentorNet results are reported based on their official github page), our method on the $60\%$ clean dataset only drops $6\%$ in accuracy compared with the classifier trained with the original clean dataset, which is the best among the reported results, and is much better than MentorNet PD. For the extremely noisy setting with $20\%$ clean samples, our approach is significantly better than MentorNet PD and close to the performance of MentorNet DD. 

\begin{table}
	\centering
	\caption{\textbf{\algname} compares with reported state-of-the-art approaches on CIFAR-10. We list their reported numbers for comparison.  \textbf{Reweight} / \textbf{MentorNet DD} require an additional 1k / 5k clean set of data to learn from (\textbf{\algname} does not require any identified clean  data). (acc.-1: accuracy of the algorithm; acc.-2 : accuracy with the full CIFAR-10 dataset.) }
	
	\scalebox{1.0}{
	\begin{tabular}{lcccc}
		\toprule 
		\quad method & clean \% & acc.-1 / acc.-2 & extra clean samples? \\
		\midrule
		1. Ours & 60\% & 86.12 / 92.40 & No\\
		2. Reweight & 60\% & 86.92 / 95.5 &  Yes, 1k \\
		4. MentorNet PD & 60\% & 77.6 / 96 & No \\
		5. MentorNet DD & 60\% & 88.7 / 96 & Yes, 5k \\
		\midrule
		6. Ours & 20\% & 42.24 / 92.40 & No \\
		7. MentorNet PD & 20\% & 28.3 / 96 & No \\
		8. MentorNet DD & 20\% & 46.3 / 96 & Yes, 5k \\
		\bottomrule
	\end{tabular} }
	\label{tab:sota}
\end{table}

\subsection{Deep generative models with mixed training data}
\label{sec:generation}

\begin{figure*}[!t]
	\begin{subfigure}{0.115\linewidth}
		\includegraphics[width=\linewidth]{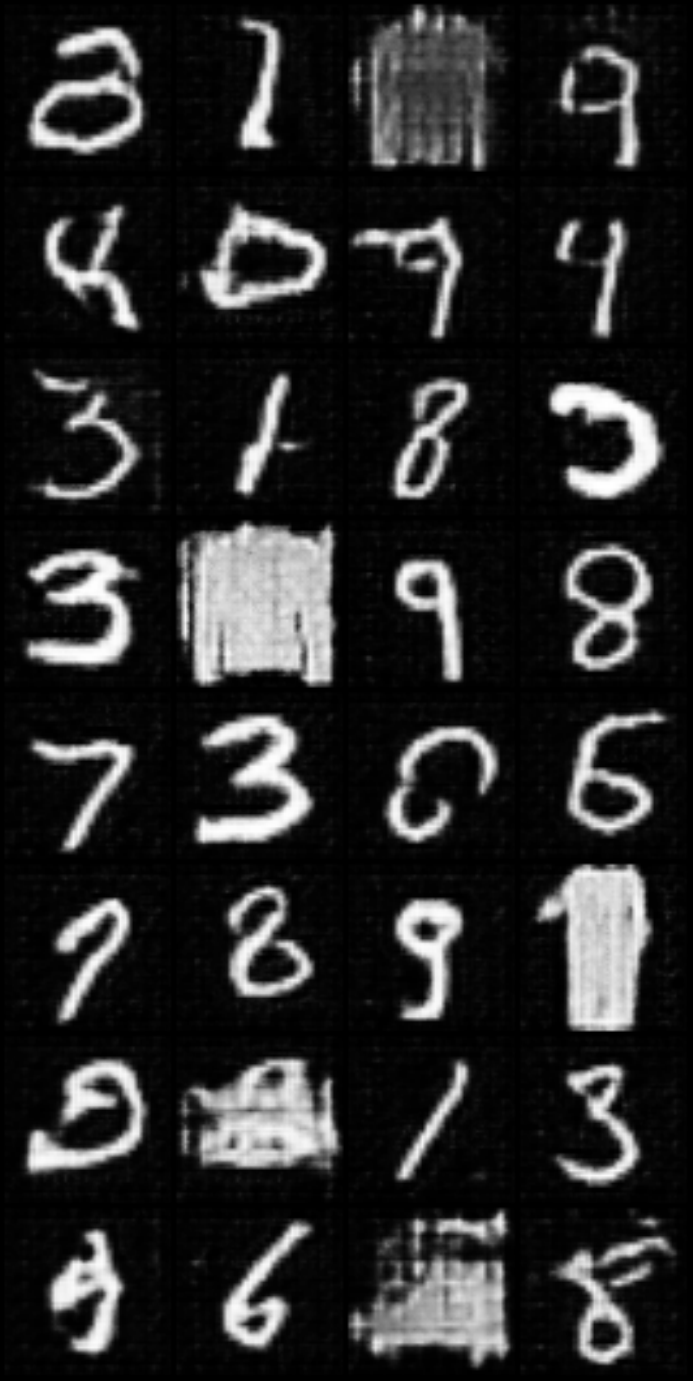}
		\caption*{baseline} \label{fig:mnist-fashion-a}
	\end{subfigure}
	\begin{subfigure}{0.115\linewidth}
		\includegraphics[width=\linewidth]{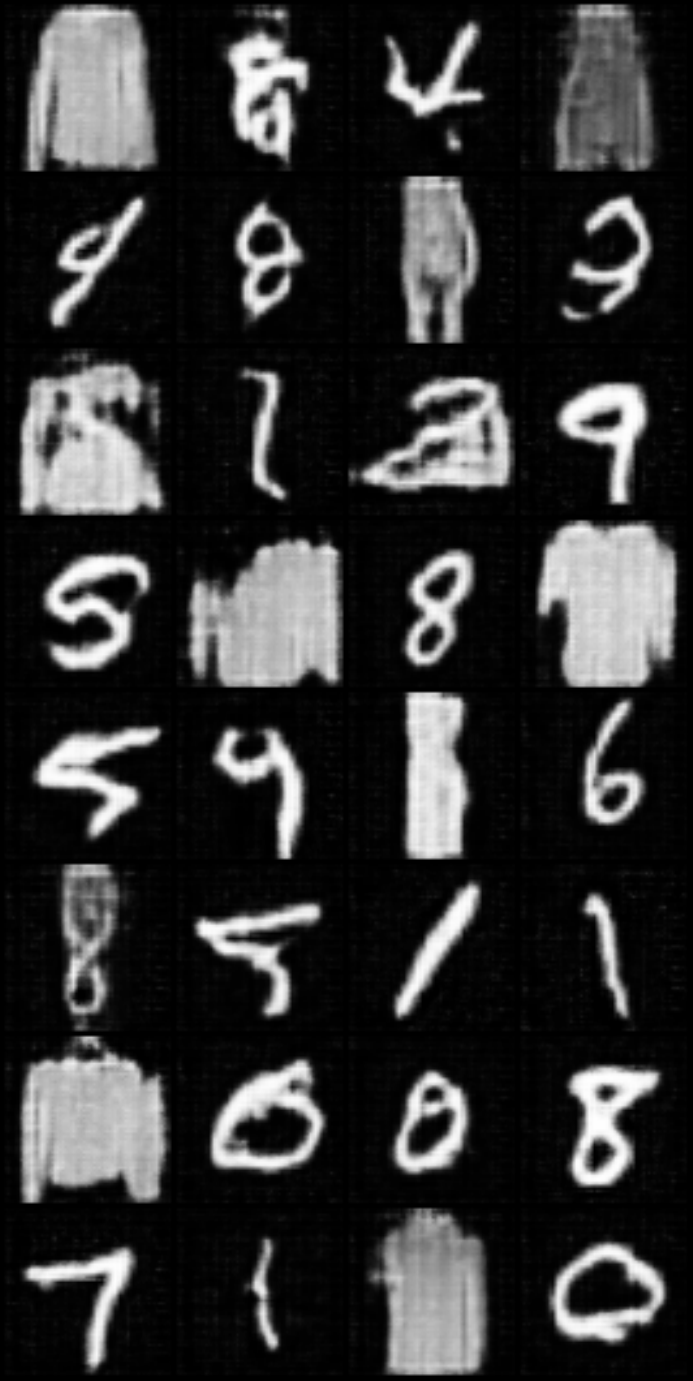}
		\caption*{1st iter.} \label{fig:mnist-fashion-b}
	\end{subfigure}
	\begin{subfigure}{0.115\linewidth}
		\includegraphics[width=\linewidth]{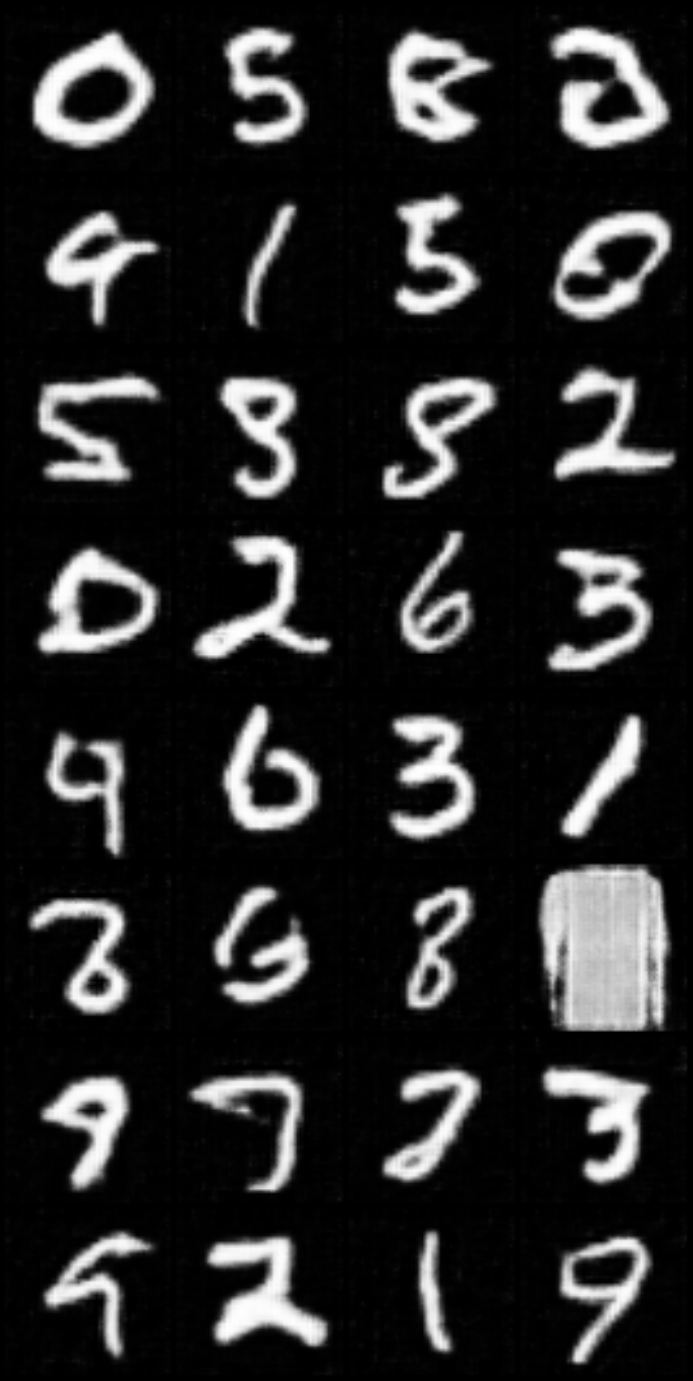}
		\caption*{3rd iter.} \label{fig:mnist-fashion-c}
	\end{subfigure}
	\begin{subfigure}{0.115\linewidth}
		\includegraphics[width=\linewidth]{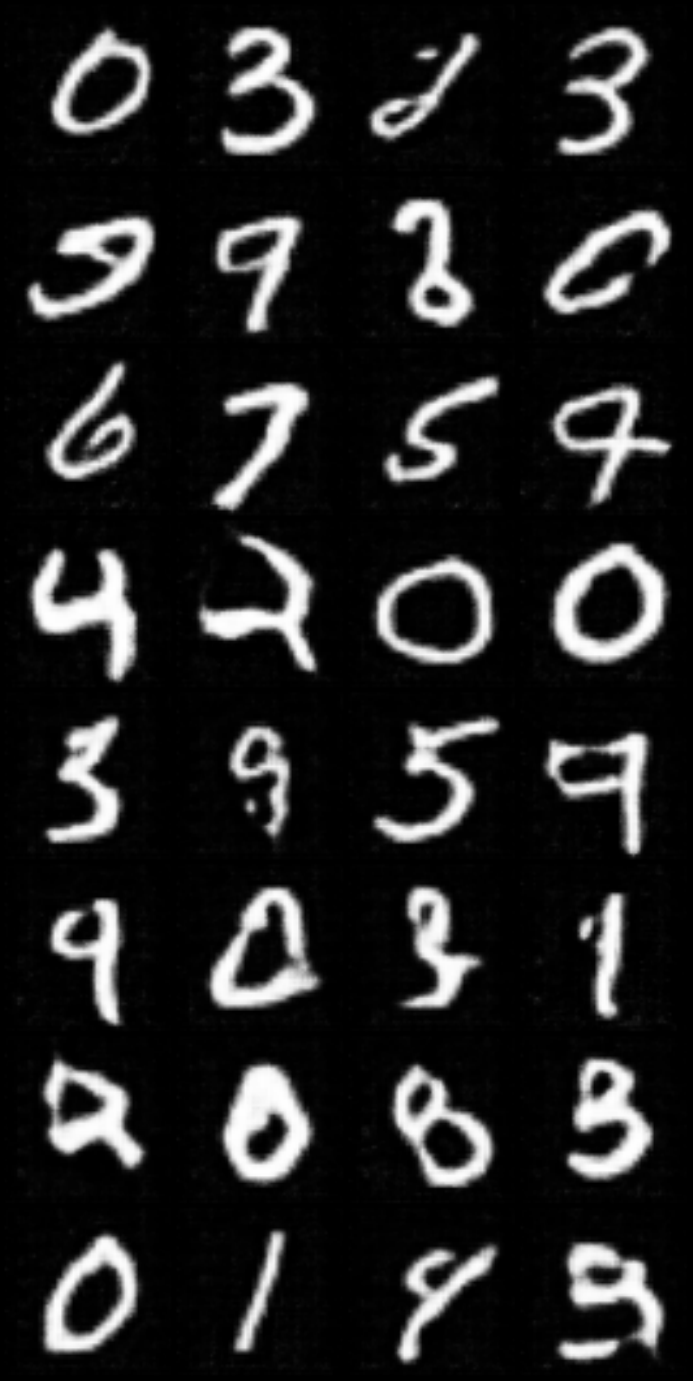}
		\caption*{5th iter.} \label{fig:mnist-fashion-d}
	\end{subfigure}
	\hspace*{\fill}
	\begin{subfigure}{0.115\linewidth}
		\includegraphics[width=\linewidth]{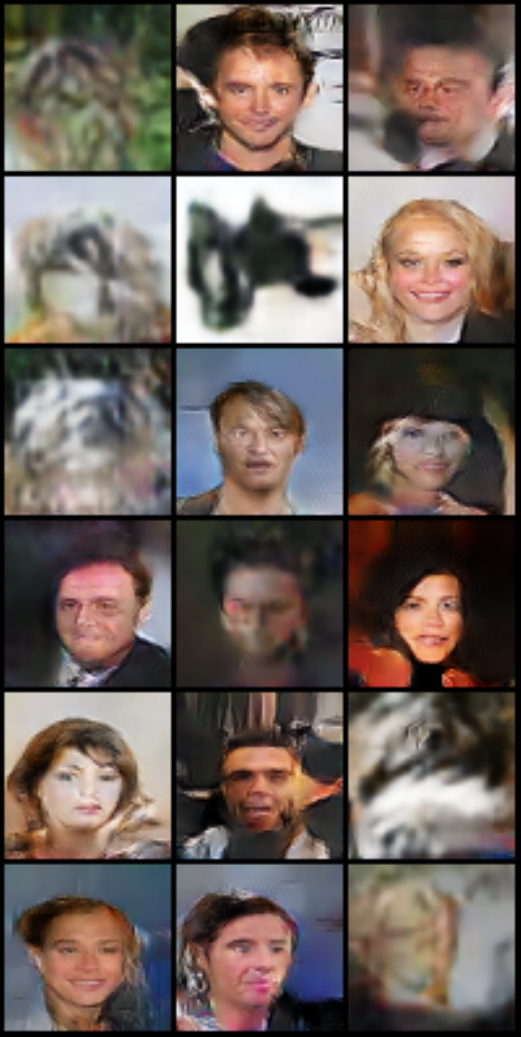}
		\caption*{baseline} \label{fig:celebA-a}
	\end{subfigure}
	\begin{subfigure}{0.115\linewidth}
		\includegraphics[width=\linewidth]{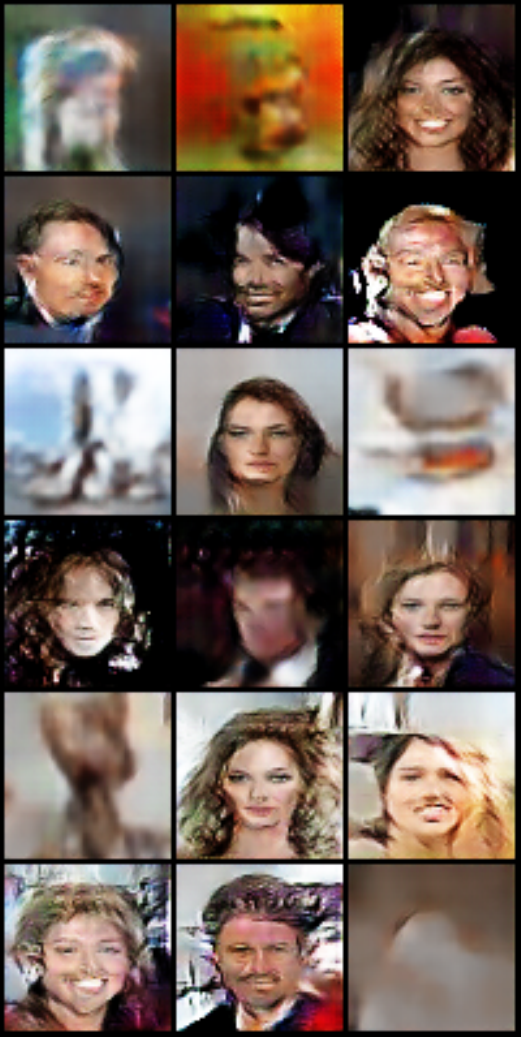}
		\caption*{1st iter.} \label{fig:celebA-b}
	\end{subfigure}
	\begin{subfigure}{0.115\linewidth}
		\includegraphics[width=\linewidth]{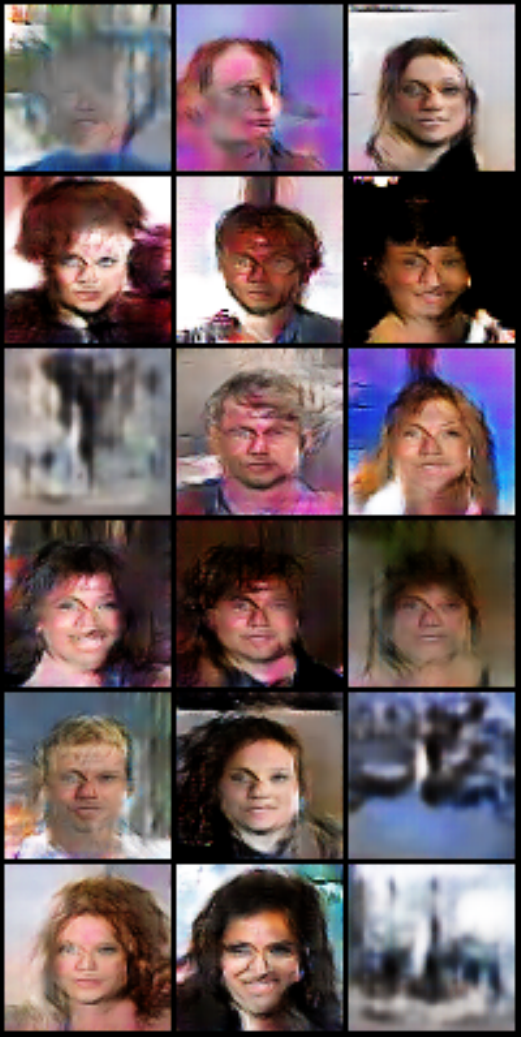}
		\caption*{3rd iter.} \label{fig:celebA-c}
	\end{subfigure}
	\begin{subfigure}{0.115\linewidth}
		\includegraphics[width=\linewidth]{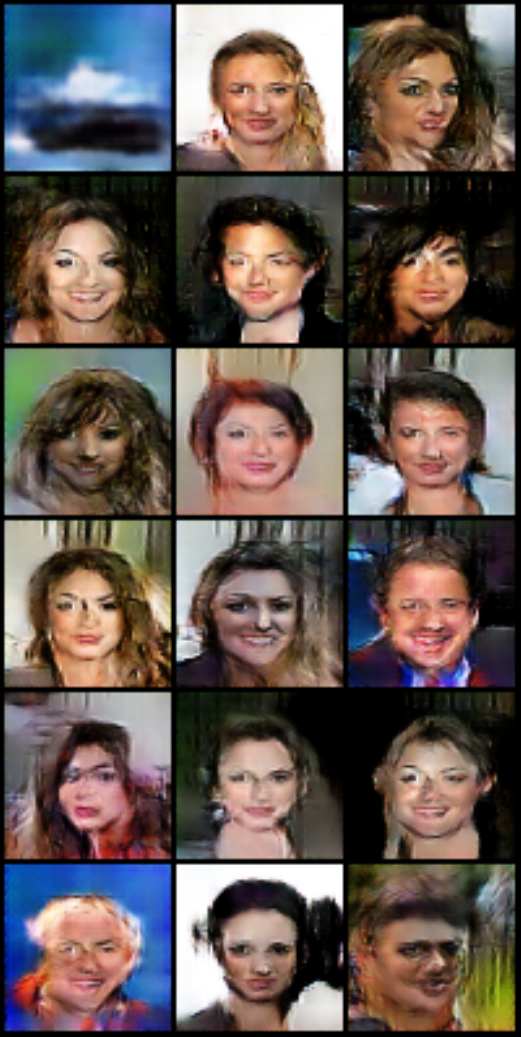}
		\caption*{5th iter.} \label{fig:celebA-d}
	\end{subfigure}
	
	\caption{ {\bf Qualitative performance of \algname for GANs:} We apply \algname to a dataset of (1) \textbf{left}:  80\% MNIST  + 20\% Fashion-MNIST; (2) \textbf{right}: 70\% CelebA  + 30\% CIFAR-10. The panels show the fake images from 32 randomly chosen (and then fixed) latent vectors, as \algname iterations update the GAN weights. Baseline is the standard training of fitting to all samples, which generates both types of images, but by the 5th iteration it hones in on digit/face-like images.}
	\label{fig:mnist-fashion}
\end{figure*}

We consider training a GAN -- specifically, the DC-GAN architecture~\cite{radford2015unsupervised}  -- to generate images similar to those from a {clean} dataset, but when the training data given to it contains some fraction of the samples from a  {bad} dataset. 
This type of bad data setting may happen very often in practice when the data collector collects a large amount of samples/images, e.g., from web search engine, to learn a generative model, and it could be difficult for the collector to find a rule to filter those incorrect samples. 
All images are unlabeled, and we do not know which training sample comes from which dataset. 
We investigated the efficacy of our approach in the following two experiments: A mixture of MNIST (clean) images with Fashion-MNIST (bad) images; A mixture of Celeb-A (clean) face images with CIFAR-10 (bad) object images.
We consider different fractions of bad samples in the training data, evaluate their effect on standard GAN training, and then the efficacy of our approach as we execute it for upto 5 iterations. 

We again use the framework of \algname, while slightly changing the two update steps. 
Recall that training a GAN consists of updating the weights of both a generator network and a discriminator network;  
our model parameters $\theta = \{\theta^D, \theta^G\}$ include the parameters of { both}  networks.  
When selecting samples, we only calculate the discriminator's loss, i.e., 
\begin{align*}
	S_t \leftarrow \arg\min_{S : |S|=\lfloor\alpha n\rfloor } \sum_{i \in S} D_{\theta^D_t}(s_i).
\end{align*}
When updating the parameter, we update both the discriminator and the generator simultaneously, as a regular GAN training process. 
Notice that for different GAN architectures, the loss function for training the discriminator varies, however, we can always find a surrogate loss: 
the loss of the discriminator for real images. 
Again, we set $\alpha$ to be $5\%$ less than the true ratio of clean samples. 

\textbf{Results:}  
Figure \ref{fig:mnist-fashion} shows qualitatively how the learned generative model performs as we iteratively learn using \algname. The generated images only contain digit-type images in the first experiment after the $5$th iteration, and similar behavior is observed for the corrupted face image. 
Table \ref{tab:gan-quant} provides  a quantitative analysis showing that \algname selects more and more clean samples iteratively. In Appendix, we also show  other simple heuristic methods fail to filter out bad data, and our approach is not sensitive to mis-specified $\alpha$. 

\begin{table}[t]
	\centering
	\caption{{\bf Generative models from mixed training data: A {\bf quantitative} measure} 
	The table depicts the ratio of the clean samples in the training data that are {\em recovered} by the discriminator when it is run on the training samples. The higher this fraction, the more effective the generator. Our approach shows significant improvements with iteration count. } 
	
	\scalebox{1.0}{
	\begin{tabular}{lcccccccccccccc}
		\toprule
		& \multicolumn{3}{c}{MNIST(clean)-Fashion(bad)} &  \multicolumn{3}{c}{CelebA(clean)-CIFAR10(bad)}\\
		\midrule
		orig &90\% & 80\% & 70\%  & 90\% & 80\% & 70\% \\
		\midrule
		iter-1&91.90\% & 76.84\% & 77.77\% &  97.12\% & 81.34\% & 75.57\% \\
		iter-2&96.05\% & 91.95\% & 79.12\% &  97.33\% & 88.11\% & 76.45\% \\
		iter-3& 99.15\% & 96.14\% & 85.66\% &  97.43\% & 89.48\% & 86.63\% \\
		iter-4&100.0\% & 99.67\% & 91.51\% &  97.53\% & 92.89\% & 82.15\% \\
		iter-5&100.0\% &100.0\% & 97.00\% &  98.14\% & 92.94\% & 94.02\% \\
		\bottomrule
	\end{tabular}}
	\label{tab:gan-quant}
\end{table}

\subsection{Defending backdoor attack}
\label{sec:backdoor}

Backdoor attack is one of the recent attacking schemes that aims at deceiving DNN classifiers. More specifically, the goal of backdoor attack is by injecting few poisoned samples to the training set, such that the trained DNNs achieve both high performance for the regular testing set and a second testing set created by the adversary whose labels are manipulated. 
We inject backdoor images using exactly the same process as described in \cite{tran2018spectral}, i.e., we pick a target class and poison 5\% of the target class images as watermarked images from other classes. See Figure \ref{fig:attack} for typical samples in a training set. Accordingly, we generate a testing set with all watermarked images, whose labels are all set to the target class. Notice that a regularly trained classifier makes almost perfect prediction in the manipulated testing set. 
We use \algname with $4$ early stopping rounds and $1$ full training round, we set $\alpha$ as $0.98$. 

\textbf{Results:} 
Our results on several randomly picked poisoned dataset are shown in Table \ref{tab:backdoor-attack}. 
The algorithm is able to filter out most of the injected samples, and the accuracy on the second testing set achieves zero. 
The early stopping rounds are very effective in filtering out watermarked samples, whereas without early stopping, the performance is poor. 

\begin{table}[ht]
		\centering
		\caption{Defending backdoor attack samples, which poisons class $a$ and make them class $b$. test-1 accuracy refers to the true testing accuracy, while test-2 accuracy refers to the testing accuracy on the test set made by the adversary.}
		\scalebox{1.0}{
		\begin{tabular}{cccccccc}
			\toprule 
			&  & {naive training} &  {with \algname } \\
			 class $a\rightarrow b$ & shape & test-1 / test-2 acc.  &  test-1 / test-2 acc.\\
			\midrule
			1 $\rightarrow $ 2 & X & 90.32 / 97.50 & 90.31 / 0.10 \\
			9 $\rightarrow $ 4 & X & 89.83 / 96.30 & 90.02 / 0.60\\
			6 $\rightarrow $ 0 & L & 89.83 / 98.10 & 89.84 / 1.30 \\
			2 $\rightarrow $ 8 & L & 90.23 / 97.90 & 89.70 / 1.20\\
			
			\bottomrule
		\end{tabular} }
		\label{tab:backdoor-attack}
\end{table}

\begin{figure}[!t]
	\begin{subfigure}{0.24\linewidth}
		\includegraphics[width=\linewidth]{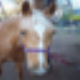}
		\caption*{a-``horse'' \\ \centering{dataset-1}} \label{fig:backdoor-a}
	\end{subfigure}
	\hfill
	\begin{subfigure}{0.24\linewidth}
		\includegraphics[width=\linewidth]{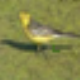}
		\caption*{b-``horse'' (bad) \\ \centering{dataset-1}} \label{fig:backdoor-b}
	\end{subfigure}
	\hfill
	\begin{subfigure}{0.24\linewidth}
		\includegraphics[width=\linewidth]{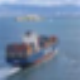}
		\caption*{c-``ship''\\ \centering{dataset-2}} \label{fig:backdoor-c}
	\end{subfigure}
	\hfill
	\begin{subfigure}{0.24\linewidth}
		\includegraphics[width=\linewidth]{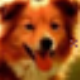}
		\caption*{d-``ship'' (bad)\\ \centering{dataset-2}} \label{fig:backdoor-d}
	\end{subfigure}
	\caption{Illustration of typical clean and backdoor samples in backdoor attacked training sets. Shown on the left are a clean ``horse'' image and a bird image with an `L'-type watermark around the center from one dataset. Shown on the right are a clean ``ship'' image and a dog image with an `X'-type watermark on the right from another  dataset. }
	\label{fig:attack}
\end{figure}

\section{Discussion}\label{sec:discussion}

We demonstrated the merit of iteratively minimize the trimmed loss, both theoretically for the simpler setting of generalized linear models, and empirically for more challenging ones involving neural networks for classification (with label noise) and GANs (with tainted samples). The \algname approach is simple, flexible and efficient enough to be applied to most modern machine learning tasks, and can serve as a strong baseline and guide when designing new approaches. It is based on the key observation that when tainted data is present, it helps to look more closely at how the loss of each training sample evolves as the model fitting proceeds. Specifically, and especially early in the model fitting process, tainted samples are seen to have higher loss in a variety of tainted settings. This is also theoretically backed up for the (admittedly simpler) generalized linear model case.

Our paper opens several interesting venues for exploration. It would be good to get a better understanding of why the evolution of loss behaves this way in neural network settings; also when it would not do so. It would also be interesting to characterize theoretically the performance for more cases beyond generalized liner models.

\clearpage 
\newpage 

\bibliography{ref}
\bibliographystyle{alpha}

\clearpage
\newpage

\appendix

\section{Property of the \estname Estimator}

\begin{proof}[Proof of Lemma 3.]
	We use standard techniques for consistency proof, similar to \cite{vcivzek2008general}. 
	First, let $f$ be the loss of a single sample, $F_n$ be the loss of sum of $n$ smallest losses over the total sample size. $f_{\lfloor\alpha n\rfloor}$ is the $\lfloor \alpha n\rfloor$-th smallest loss. We can re-write $F_n$ into the following two terms:
	\begin{align}
		F_n(\theta) = & \frac{1}{n} \sum_{i=1}^{n} f(s_i ; \theta) \cdot \mathbb{I}\left\{ f(s_i ; \theta) \le f_{\lfloor \alpha n\rfloor}(\theta) \right\} \nonumber \\
		= &\frac{1}{n} \sum_{i=1}^{n} f(s_i ; \theta) \cdot \left( \mathbb{I}\left\{ f(s_i ; \theta) \le f_{\lfloor \alpha n\rfloor}(\theta) \right\} - \mathbb{I}\left\{ f(s_i ; \theta) \le D_{\theta}^{-1}(\alpha)  \right\} \right) \label{eqt:app-1} \\
		& +\frac{1}{n} \sum_{i=1}^{n} f(s_i ; \theta) \cdot \mathbb{I}\left\{ f(s_i ; \theta) \le D_{\theta}^{-1}(\alpha) \right\},  \label{eqt:app-2}
	\end{align}
	where $D_\theta$ is the distribution function of $f_\theta(s)$, and $D_\theta^{-1}$ is its inverse function, which calculates the quantile value. 
	On the other hand, define $F$ to be the expected average trimmed loss, i.e., 
	\begin{align}
	 	F(\theta) = & \mathbb{E}\left[ f(s_i ; \theta) \cdot \mathbb{I} \left\{  f(s_i ; \theta) \le D_{\theta}^{-1}(\alpha) \right\} \right]. \label{eqt:app-3}
	\end{align}
	Then, the difference between $F_n(\theta)$ and $F(\theta)$ can be separated into two terms: the first term is the difference between (\ref{eqt:app-2}) and (\ref{eqt:app-3}), which asymptotically goes to zero due to the law of large numbers;
	on the other hand, the term (\ref{eqt:app-1}) goes to zero  because of the convergence of order statistics to the quantile. See  \cite{vcivzek2008general} for showing the consistency of both  terms under the regularity conditions. 
	
	By definition, \estname $\hat{\theta}^{(\mathtt{\estname})}$ satisfies  $\Pr\left[ F_n(\hat{\theta}^{(\mathtt{\estname})}) < F_n(\theta^\star) \right] = 1$. For any $\epsilon > 0$, 
	\begin{align}
		1 = & \Pr \left[ F_n(\hat{\theta}^{(\mathtt{\estname})}) < F_n(\theta^\star) \right] \nonumber  \\
		= & \Pr  \left[ F_n(\hat{\theta}^{(\mathtt{\estname})}) < F_n(\theta^\star), \hat{\theta}^{(\mathtt{\estname})} \in \mathcal{U}(\theta^\star, \epsilon) \right] + \Pr  \left[ F_n(\hat{\theta}^{(\mathtt{\estname})}) < F_n(\theta^\star), \hat{\theta}^{(\mathtt{\estname})} \in \mathcal{B} \backslash \mathcal{U}(\theta^\star, \epsilon) \right] \nonumber \\
		\le & \Pr \left[ \hat{\theta}^{(\mathtt{\estname})} \in \mathcal{U}(\theta^\star, \epsilon) \right] + \Pr \left[ \inf_{\theta \in \mathcal{B} \backslash \mathcal{U}(\theta^\star, \epsilon)} F_n(\theta) < F_n(\theta^\star) \right],  \label{eqt:app-7}
	\end{align}
	where in the last inequality, we use the fact that the probability measure on a set is no less than the probability measure on its subset. Notice that our goal is to show $\hat{\theta}^{(\mathtt{\estname}) }$ is in $\mathcal{U}(\theta^\star, \epsilon)$ with probability $1$. This is true as long as the second term is zero. 
	The second term in the above can be controlled by
	\begin{align}
		& \Pr \left[ \inf_{\theta \in \mathcal{B} \backslash \mathcal{U}(\theta^\star, \epsilon)} F_n(\theta) < F_n(\theta^\star) \right] \nonumber \\
		= & \Pr \left[ \inf_{\theta \in \mathcal{B} \backslash \mathcal{U}(\theta^\star, \epsilon)} \left[ F_n(\theta) - F(\theta) + F(\theta) \right] < F_n(\theta^\star) \right] \nonumber  \\
		\le & \Pr \left[ \inf_{\theta \in \mathcal{B} \backslash \mathcal{U}(\theta^\star, \epsilon)} \left[ F_n(\theta) - F(\theta)  \right] < F_n(\theta^\star) - \inf_{\theta \in \mathcal{B} \backslash \mathcal{U}(\theta^\star, \epsilon)} F(\theta)  \right] \label{eqt:app-4} \\
		\le & \Pr \left[ \sup_{\theta \in \mathcal{B}} \left\vert F_n(\theta) - F(\theta) \right\vert > \inf_{\theta \in \mathcal{B} \backslash \mathcal{U}(\theta^\star, \epsilon)} F(\theta) - F_n(\theta^\star) \right] \label{eqt:app-5} \\
		\le & \Pr \left[ 2 \sup_{\theta \in \mathcal{B}} \left\vert F_n(\theta) - F(\theta) \right\vert > \inf_{\theta \in \mathcal{B} \backslash \mathcal{U}(\theta^\star, \epsilon)} F(\theta) - F(\theta^\star) \right], \label{eqt:app-6}
	\end{align}
	where (\ref{eqt:app-4}) is due to triangle inequality, in (\ref{eqt:app-5}), we flip the sign on both sides and upper bound the difference by the abstract value. (\ref{eqt:app-6}) again uses triangle inequality, in order to separate the population loss on $\theta$ and the sample loss on $\theta^\star$. 
	As we have discussed at the beginning, under regularity conditions, $F_n(\theta) - F(\theta)$ goes to zero asymptotically. More specifically, for any $\epsilon > 0$, 
	 $\Pr \left[ \sup_{\theta \in \mathcal{B}} |F_n(\theta) - F(\theta)| > \frac{\delta(\epsilon)}{2} \right] \rightarrow 0$ as $n\rightarrow +\infty$. 
	On the other hand, 
	 $\Pr \left[ \inf_{\theta \in \mathcal{B} \backslash \mathcal{U}(\theta^\star, \epsilon)} F(\theta) - F(\theta^\star) < \delta(\epsilon) \right] =0 $,
	which is given by the idenfication condition. Combining with (\ref{eqt:app-7}), and  triangle inequality, we have $|S_n(\hat{\theta}^{(\mathtt{\estname})}) - S_n(\theta^\star)| \rightarrow 0$ with probability $1$, as $n\rightarrow \infty$.
\end{proof}

\section{Clarification of the \algname  Algorithm}
For different settings, we use the same procedure as described in Algorithm 1 and 2 (in the main paper), but may select different hyper-parameters. 
We summarize the alternatives we use for all the settings as follows:  
\begin{enumerate}
	\item[(a)]  In the linear setting, choosing a large $M$ with small step size $\eta$ corresponds to finding the closed form solution, which is the setting we analyze;
	\item[(b)] For generalized linear setting, we  analyze for $M=1$ and $N=|S|$, which corresponds to a single full gradient update per round;
	\item[(c)] For all experiments with DNNs, we run re-initialization for every round of update to make it harder to stuck at bad local minimum; 
	\item[(d)] For training generative model using GANs, we use the loss of  discriminator's output in step $4$ in Algorithm 1, and use the joint loss of both the generator and the discriminator in Algorithm 2;
	\item[(e)] For CIFAR-10 classification tasks with (i) bad labels, (ii) backdoor samples, we choose smaller $M$ for the first $4$ rounds, which corresponds to early stopping. 
	As motivated in Section 1 (in the main paper), early stopping may help us better filter out bad samples since the later rounds may overfit on them. 
\end{enumerate}

\section{Proofs for ILTM Algorithm}

\begin{proof}[Proof of Lemma \ref{thm:alg}]
	Let $\theta_t$ be the  learned parameter at round $t$, and $\theta_{t+1}$ be the learned parameter in the next round, following Algorithm 1. 
	More specifically, a subset $S_t$ of size $\alpha n$ with the smallest losses $(y_i - \theta_t^\top \cdot \phi(x_i))^2$ is selected. 
	$\theta_{t+1}$ is the minimizer on the selected set of sample losses. Denote $W_t$ as the diagonal matrix whose diagonal entry $W_{t, ii}$ equals $1$ when the $i$-th sample is in set $S_t$, otherwise $0$. Then, assume that we take infinite steps and reach  the optimal solution (we will discuss how to extend this to arbitrary $M_t$  with small step size later), we have : 
	\begin{align*}
	\theta_{t+1} = \left(\Phi(X)^\top W_t \Phi(X)\right)^{-1}\Phi(X)^\top W_t y,
	\end{align*}
	where $\Phi(X)$ is an $n\times d$ matrix,  whose $i$-th row is $\phi(x_i)^\top$, and we have used the  fact that $W_t^2 = W_t$.
	Remind that for the feature matrix $\Phi(X)$, we have defined 
	\begin{align*}
	\psi^{-}(k) = & \min_{W:  W\in \mathcal{W}_k} \sigma_{\min} \left( \Phi(X)^\top W\Phi(X) \right),  \\
	\psi^{+}(k) = & \max_{W: W\in \mathcal{W}_k} \sigma_{\max} \left(  \Phi(X)^\top W \Phi(X) \right),
	\end{align*}
	which will be used in the later analysis. For $\Phi(X)$ whose every row follows i.i.d. sub-Gaussian random vector, by using concentration of the spectral norm of Gaussian matrices, and uniform bound, $\Phi(X)$ is a regular feature matrix, see, e.g., Theorem 17 in \cite{bhatia2015robust}, and other literatures~\cite{davenport2009simple}. 
	
	On the other hand, denote $W^\star$ as the ground truth diagonal matrix for the samples, i.e., $W_{ii}^\star = 1$ if the $i$-th sample is a clean sample, otherwise $W_{ii}^\star = 0$. 
	Accordingly, define $S^\star$ as the ground truth set of clean samples. For clearness of the presentation, we may drop the subscript $t$ when there is no ambiguation. 
	For bad samples, the output is written in the form of $y_i = r_i + e_i$, where $e_i$ represents the observation noise, and $r_i$ depends on the specific setting we consider (we will discuss more in later Theorems).  Under this general representation, we can re-write the term $\theta_{t+1}$ as
	\begin{align*}
		\theta_{t+1}  = &  \left(\Phi(X)^\top W \Phi(X)\right)^{-1} \Phi(X)^\top W\left(W^\star \Phi(X) \theta^\star + (I-W^\star) r + e \right)  \\
		= & \theta^\star + \left(\Phi(X)^\top W \Phi(X)\right)^{-1} \left( \Phi(X)^\top WW^\star \Phi(X) \theta^\star + \Phi(X)^\top W r - \Phi(X)^\top W W^\star r - \Phi(X)^\top W\Phi(X)\theta^\star + \Phi(X)^\top W e \right) \\
		= & \theta^\star + \left(\Phi(X)^\top W \Phi(X)\right)^{-1}\Phi(X)^\top \left( WW^\star - W\right)\left(\Phi(X)\theta^\star - r - e\right) + \left(\Phi(X)^\top W \Phi(X)\right)^{-1} \Phi(X)^\top WW^\star e,
	\end{align*}
	by basic linear algebra. Therefore, the $\ell_2$ distance between the learned parameter and ground truth parameter can be bounded by:
	\begin{align*}
		& \|\theta_{t+1} - \theta^\star \|_2 \\
		=& \left\| \left(\Phi(X)^\top W \Phi(X)\right)^{-1}\Phi(X)^\top \left( WW^\star - W\right)\left(\Phi(X)\theta^\star - r - e\right) + \left(\Phi(X)^\top W \Phi(X)\right)^{-1} \Phi(X)^\top WW^\star e \right\|_2 \\
		\le & \underbrace{ \left\| \left(\Phi(X)^\top W \Phi(X)\right)^{-1} \right\|_2}_{\mathcal{T}_1} \cdot \left( \underbrace{ \left\| \Phi(X)^\top \left( WW^\star - W\right)\left(\Phi(X)\theta^\star - r - e\right) \right\|_2}_{\mathcal{T}_2} + \underbrace{\left\| \Phi(X)^\top WW^\star e \right\|_2}_{\mathcal{T}_3} \right),
	\end{align*}
	where basic spectral norm inequalities and triangle inequalities. 
	For the term $\mathcal{T}_1$, notice that $W$ selects $\alpha n$ rows of $\Phi(X)$, i.e., $\mathtt{Tr}(W) = \alpha n$. Therefore, $\mathcal{T}_1 \le \frac{1}{\psi^{-}(\alpha n)}$.
	
	Next, the term $\mathcal{T}_2$ can be bounded as:
	\begin{align}
		\mathcal{T}_2^2 = & \left\| \Phi(X)^\top \left( W - WW^\star\right)\left(\Phi(X)\theta^\star - r - e\right) \right\|_2^2 \nonumber \\
		= & \left(\Phi(X)\theta^\star - r - e\right)^\top \left[\left( W - WW^\star\right) \Phi(X) \Phi(X)^\top \left( W - WW^\star\right) \right]  \left(\Phi(X)\theta^\star - r - e\right) \nonumber  \\
		\le & 2 \left(\Phi(X)\theta^\star - \Phi(X)\theta_t\right)^\top \left[\left( W - WW^\star\right) \Phi(X) \Phi(X)^\top \left( W - WW^\star\right) \right]\left(\Phi(X)\theta^\star - \Phi(X)\theta_t\right) \nonumber \\
		& + 2 \left(\Phi(X)\theta_t - r - e\right)^\top \left[\left( W - WW^\star\right) \Phi(X) \Phi(X)^\top \left( W - WW^\star\right) \right]\left(\Phi(X)\theta_t - r - e\right) \nonumber \\
		\le & 2 \sigma_{\max}\left(\Phi(X)^\top (W-WW^\star) \Phi(X)\right)^2 \left\|\theta^\star - \theta_t\right\|_2^2 \label{eqt:app-8} \\
		& + 2  \underbrace{ \left(\Phi(X)\theta_t - r - e\right)^\top  \left[ \left( W - WW^\star\right) \Phi(X) \Phi(X)^\top \left( W - WW^\star\right) \right] \left(\Phi(X)\theta_t - r - e\right)}_{ \varphi(S_t, S^\star, \|\theta^\star - \theta_t\|_2)^2}. \label{eqt:app-9}
	\end{align}
	The last term (\ref{eqt:app-9}) is defined as $\varphi_t:= \varphi(S_t, S^\star, \|\theta^\star - \theta_t\|_2) = \left\| \sum_{i\in S\backslash S^\star} (\phi(x_i)^\top \theta_t - r_i - e_i) \phi(x_i) \right\|_2$. For the term (\ref{eqt:app-8}), let $|S_t\backslash S^\star|$ be the number of bad samples in $S_t$. Then, the eigenvalue is bounded by $\psi^{+}(|S_t\backslash S^\star|)$. 
	
	The term $\mathcal{T}_3$ can be bounded as:
	\begin{align*}
		\mathcal{T}_3^2 = \left\| \Phi(X)^\top WW^\star e \right\|_2^2 \le e^\top \Phi(X)\Phi(X)^\top e = \sum_{i=1}^d \left( \sum_{j=1}^n e_j \phi(x_j)_{i} \right)^2 \le c \sum_{i=1}^n \|\phi(x_i)\|_2^2 \log n \sigma^2,
	\end{align*}
	where the last inequality holds with high probability, and all the randomness comes from the measurement noise $e$. The last inequality is based on the  sub-exponential concentration property. 

	Then, as a summary, combining the results for all three terms, we have: 
	\begin{align*}
		\|\theta^\star - \theta_{t+1}\|_2 \le \frac{ \sqrt{2} \psi^{+}(|S\backslash S^\star|)}{\psi^{-}(\alpha n)} \|\theta^\star - \theta_t\|_2 + \frac{\sqrt{2} \varphi(S_t, S^\star, \|\theta^\star- \theta_t\|_2)}{\psi^{-}(\alpha n)} + \frac{c\sqrt{\sum_{i=1}^n \|\phi(x_i)\|_2^2 \log n  }}{\psi^{-}(\alpha n)} \sigma.
	\end{align*}
\end{proof}

\paragraph{Discussion on finite $M_t$} As we mentioned before, for the simplicity of the result, we consider $\theta_{t+1}$ as a full update on the subset of samples. However, based on this current framework, we can also analyze for finite $M_t$, with small step size $\eta$. The key idea is that in the linear setting, we can connect the updated parameter at each epoch with a closed form solution to a penalized minimization problem. More specifically, accordng to~\cite{suggala2018connecting}, define
\begin{align*}
	\dot{\theta}(t) := \frac{d}{dt} \theta(t) = - \nabla f(\theta(t)), \theta(0) = \theta_0,
\end{align*}
and
\begin{align*}
	\underline{\theta}(\nu) = \arg\min_{\theta}  f(\theta) + \frac{1}{2\nu} \|\theta - \theta_0\|_2^2,
\end{align*}
where $f(\theta) = \frac{1}{2|S|}\sum_{i\in S} (y_i - \phi(x_i)^\top \theta)^2$. Then, ${\theta}(t)$ and $\underline{\theta}(\nu)$ have the following relationship:
\begin{align*}
	\|\theta(t) - \underline{\theta}(\nu(t)) \|_2 \le \frac{\|\nabla f(\theta_0)\|_2}{m} \left( e^{-mt} + \frac{c}{1-c-e^{cMt}} \right),
\end{align*}
where $\nu(t) = \frac{1}{cm}\left(e^{cMt}-1 \right)$, for $m = \sigma_{\min} (\frac{1}{|S|} \Phi(X)^\top W\Phi(X))$, $M = \sigma_{\max}(\frac{1}{|S|} \Phi(X)^\top W\Phi(X))$, $c = \frac{2m}{M+m}$. Since $\underline{\theta}(\nu)$ has a closed form solution in this linear setting, by connecting $\theta^{t+1}$ with $\underline{\theta}$, we are able to bound $\theta^{t+1}$ using similar proof technique as above.

\begin{proof}[Proof of Lemma \ref{thm:nonlinear}]
	Define $F:\mathbb{R}^n \rightarrow \mathbb{R}^n$ as an entry-wise  $f(\cdot )$-operation. 
	\begin{align*}
	\theta_{t+1} = & \theta_t - \frac{\eta}{\alpha n} \sum_{i\in S_t} \left( f\left( \phi(x_i)^\top \theta_t \right) - y_i \right)\cdot f'\left(\phi(x_i)^\top \theta_t \right)\cdot \phi(x_i) \\
	= & \theta_t - \frac{\eta}{\alpha n} \Phi(X)^\top \mathtt{Diag}\left( F'\left( \Phi(X) \theta_t \right) \right)W_t \left( F\left(\Phi(X)\theta_t\right) - y\right) \\
	= & \theta_t - \frac{\eta}{\alpha n} \Phi(X)^\top \mathtt{Diag}\left( F'\left( \Phi(X) \theta_t \right) \right)W_t \left( F\left(\Phi(X)\theta_t\right) - W^\star F\left( \Phi(X)\theta^\star \right) - \left(I-W^\star \right)(r+e) - W^\star e \right) \\
	= & \theta_t - \frac{\eta}{\alpha n} \Phi(X)^\top \mathtt{Diag}\left( F'\left( \Phi(X) \theta_t \right) \right)W_t \left( F\left(\Phi(X)\theta_t\right) - W^\star F\left( \Phi(X)\theta^\star \right) -(I-W^\star)F\left(\Phi(X)\theta^\star\right) \right) \\
	& -  \frac{\eta}{\alpha n} \Phi(X)^\top \mathtt{Diag}\left( F'\left( \Phi(X) \theta_t \right) \right)W_t \left( (I-W^\star)F\left(\Phi(X)\theta^\star\right) - \left(I-W^\star \right)(r+e) - W^\star e \right) \\
	= & \theta_t - \frac{\eta}{\alpha n} \Phi(X)^\top \mathtt{Diag}\left( F'\left( \Phi(X) \theta_t \right) \right)  W_t \left( F\left(\Phi(X)\theta_t\right) - F\left(\Phi(X)\theta^\star\right) \right)  \\
	& - \frac{\eta}{\alpha n} \Phi(X)^\top \mathtt{Diag}\left( F'\left( \Phi(X) \theta_t \right) \right) \left(W_t - W_t W^\star \right) \left( F\left(\Phi(X)\theta^\star \right) - r- e\right) \\
	& + \frac{\eta}{\alpha n} \Phi(X)\mathtt{Diag}\left( F'\left( \Phi(X) \theta_t \right) \right) W_t W^\star e. 
	\end{align*}
	We simplify the notation using $H_t\triangleq \mathtt{Diag}\left( F'\left( \Phi(X) \theta_t \right) \right)$. Also, by mean value theorem, for any $a,b$, there exists some $c\in [a,b]$, such that $\frac{f(b)-f(a)}{b-a} = f'(c)$. Therefore, for the term $F(\Phi(X)\theta_t) - F(\Phi(X)\theta^\star)$, there exists a diagonal matrix $C_t$, such that $F(\Phi(X)\theta_t) - F(\Phi(X)\theta^\star) = C_t \Phi(X)\left(\theta_t - \theta^\star\right)$. Therefore, we have
	\begin{align*}
	\left\|\theta_{t+1} - \theta^\star \right\|_2 \le & \left( \underbrace{ 1 - \frac{\eta}{\alpha n} \Phi(X)^\top H_t W_t C_t \Phi(X)}_{\mathcal{U}_1} \right) \left\| \theta_t - \theta^\star\right\|_2 + \frac{\eta}{\alpha n} \underbrace{\left\| \Phi(X)^\top H_t (W_t - W_tW^\star) \left( F(\Phi(X)\theta^\star) - r - e\right) \right\|_2}_{\mathcal{U}_2} \\
	& + \frac{\eta}{\alpha n} \underbrace{ \left\|\Phi(X)H_t W_t W^\star e\right\|_2}_{\mathcal{U}_3}.
	\end{align*}
	Here, 
	\begin{align*}
	\mathcal{U}_1 \le  1 - \eta a^2 \frac{\psi^{-}(\alpha n)}{\alpha n}, \mathcal{U}_3 \le  b\xi_t \sigma.
	\end{align*}
	For $\mathcal{U}_2$, define $\tilde{\phi}_t$ similar to $\phi_t$:
	\begin{align*}
		\tilde{\varphi}_t = \left\| \sum_{i\in S_t\backslash S^\star} \left( w(\phi(x_i)^\top \theta^\star) - r_i-e_i\right) w'(\phi(x_i)^\top \theta^\star) \phi(x_i)\right\|. 
	\end{align*}
	As a result, we have:
	\begin{align*}
	\|\theta_{t+1} - \theta^\star \|_2 \le \left( 1 - \frac{\eta}{\alpha n} {a^2 \psi^{-}(\alpha n)}  \right)  \|\theta_t - \theta^\star \|_2 + {\eta}\frac{  \tilde{\varphi}_t   + \xi_t b \sigma }{\alpha n}.
	\end{align*}
\end{proof}

\begin{proof}[Proof of Theroem \ref{thm:random}]
	Now we consider recovery in the context of aribitrary corrupted output, and random noise setting. 
	
	Notice that since samples in $S_t\backslash S^\star$ are selected because of smaller losses, and $\alpha < \alpha^\star$,  there exists a permutation matrix $P$, such that the following inequality holds element-wise: 
	$$
	(W-WW^\star) \left\vert \Phi(X) \theta_t - r - e \right\vert \le (W-WW^\star) P \left\vert \Phi(X)(\theta_t - \theta^\star) - e\right\vert. 
	$$
	
	Accordingly, given a valid permutation matrix $P$, $\phi_t$ is further bounded by
	\begin{align}
		& \phi(S_t, S^\star, \|\theta^\star - \theta_t\|_2)^2 \nonumber \\ 
		\le & \left(\Phi(X)(\theta_t - \theta^\star) - e\right)^\top N P^\top (W-WW^\star) \Phi(X) \Phi(X)^\top  (W-WW^\star)  P N \left(\Phi(X)(\theta_t - \theta^\star) - e\right) \nonumber \\
		\le & 2 (\theta_t - \theta^\star)^\top \Phi(X)^\top N P^\top (W-WW^\star) \Phi(X) \Phi(X)^\top  (W-WW^\star)  P N \Phi(X)  (\theta_t - \theta^\star) \label{eqt:app-10} \\
		& + 2 e^\top N P^\top (W-WW^\star) \Phi(X) \Phi(X)^\top  (W-WW^\star)  P N e \label{eqt:g-1}\\
		\le & 2 \psi^{+}(|S_t\backslash S^\star|)^2 \left\| \theta_t - \theta^\star \right\|_2^2 + 2c  \psi^{+}(|S_t\backslash S^\star|) n  \sigma^2  \label{eqt:g-2},
	\end{align}
	where the last inequality (\ref{eqt:g-2}) holds  with high probability. Here, $N$ is some diagonal matrix whose  entries are either $1$ or $-1$. 
	More specifically, (\ref{eqt:app-10}) can be bounded by $2\tilde{\sigma}^2 \|\theta_t -\theta^\star\|_2^2$, where $\tilde{\sigma}$ is the top singular value of the matrix $$\Phi(X)^\top (W-WW^\star)PN\Phi(X).$$ 
	Equivalently, it can be written as
	$$
	\tilde{\sigma} = \max_{u,v: \|u\|_2 = \|v\|_2=1} u^\top \Phi(X)^\top (W-WW^\star) PN \Phi(X) v.
	$$
	If we denote $\Phi(X)v$ and $\Phi(X)u$ as $\tilde{v}$, $\tilde{u}$ respectively, then 
	$$\tilde{\sigma} \le
	\sum_{i=1}^{|S\backslash S^\star|} \left\vert \tilde{u}_{r_i} \tilde{v}_{t_i}\right\vert \le \max \left\{ \sum_{i=1}^{|S\backslash S^\star|} \tilde{u}_{r_i}^2, \sum_{i=1}^{|S\backslash S^\star|} \tilde{v}_{t_i}^2 \right\},
	$$ 
	for some sequences $\{r_i\}$ and $\{t_i\}$. 
	This shows that the top singular value is indeed bounded by  
	$$
	\max \left\{
	\sigma_{\max}\left(\Phi(X)^\top (W-WW^\star) \Phi(X) \right), \sigma_{\max} \left(\Phi(X)^\top NP^\top (W-WW^\star) PN\Phi(X) \right) 
	\right\},
	$$ 
	which is bounded by $\psi^{+}(|S_t\backslash S^\star|)$, since both $W-WW^\star$ and $NP^\top (W-WW^\star) PN$ have $\mathtt{Tr}(W-WW^\star)$ non-zero entries in the diagonal. 
	
	The term (\ref{eqt:g-1}) is bounded because of the feature regularity property. Notice that $(W-WW^\star)\Phi(X) \Phi(X)^\top (W-WW^\star)$ has the same non-zero eigenvalues as $\Phi(X)^\top (W-WW^\star) \Phi(X)$. 
	
	Therefore, with high probability, 
	\begin{align*}
		\phi_t \le & \sqrt{2\psi^{+}(|S_t\backslash S^\star|)^2\|\theta^\star - \theta_t \|_2^2 + 2c  \psi^{+}(|S_t\backslash S^\star|) n  \sigma^2  } \\
		\le & \sqrt{2}\psi^{+}(|S_t\backslash S^\star|)\|\theta^\star - \theta_t \|_2 + \sqrt{2c  \psi^{+}(|S_t\backslash S^\star|) n  }  \sigma.
	\end{align*}
	Combining previous results, with high probability, we have
	\begin{align} \label{eqt:app-11}
		\|\theta^\star - \theta_{t+1} \|_2 \le \underbrace{ \frac{2\sqrt{2}\psi^{+}(|S_t\backslash S^\star |)}{\psi^{-}(\alpha n)} }_{\kappa_t} \|\theta^\star - \theta_t \|_2 +\frac{\sqrt{2c  \psi^{+}(|S_t\backslash S^\star|) n  }}{\psi^{-}(\alpha n)} \sigma+ \frac{c\sqrt{\sum_{i=1}^n \|\phi(x_i)\|_2^2 \log n  }}{\psi^{-}(\alpha n)} \sigma.
	\end{align}
	The above result holds for both the setting of random output and arbitrary corruption setting. 
	For arbitrary output setting, since $\psi^{+}(|S_t\backslash S^\star|)$ can be upper bounded by $\mathcal{O}(n)$, we have:
	\begin{align*}
		\|\theta^\star - \theta_{t+1}\|_2 \le \frac{1}{2} \|\theta^\star - \theta_t\|_2 + c \sigma + \frac{c \xi_t}{n}\sigma. 
	\end{align*}
	In the random output setting, however, in fact we can calculate how the quantity $|S_t\backslash S^\star|$ changes, and have a better characterization of the convergence. Based on Theorem \ref{lem:help}, we have:
	\begin{align*}
		\kappa_t \le c \left\{\sqrt{\|\theta_t - \theta^\star\|_2^2 + \sigma^2} \vee \frac{\log n}{n}\right\},
	\end{align*}
	for any fixed $\theta_t$. One can use a standard $\epsilon$-net argument to show that the above indeed holds for any $\theta_t$. 
	Therefore, for the case of random output corruption, 
	\begin{align*}
	\|\theta^\star - \theta_{t+1}\|_2 \le \kappa_t \|\theta^\star - \theta_t\|_2 + c \sqrt{ \kappa_t } \sigma + \frac{c \xi_t}{n}\sigma, 
	\end{align*}
	for $\kappa_t \le c \{\sqrt{\|\theta_t - \theta^\star\|_2^2 + \sigma^2} \vee \frac{\log n}{n}\} $.
\end{proof}

\begin{proof}[Proof of Theorem \ref{theorem:se}]
	In the context of mixed model setting, we are interested in when the algorithm will find the component that it is  closest to. 
	The proof outline is similar to Theorem 7. 
	However, for the case of mixture output, two parts in (\ref{eqt:app-11}) need re-consideration: the first part is to show that there is an $\Omega(n)$ lower bound for $\psi^{-}(\alpha n)$ for arbitrary constant $\alpha$. 
	Notice that in Theorem 17 of \cite{bhatia2015robust}, $\alpha$ can not be too small, e.g., $0.1$. The main idea of their proof was to use a uniform bound over all possible $W$s, which depends on $n$. However, we take another route and using $\epsilon$-net argument on the parameter space. 
	Notice that we can choose an $\epsilon$-net in $\mathbb{R}^d$, which includes $(1+\frac{2}{\epsilon})^d$ points~\cite{vershynin2016high}. For any fixed $\theta$, notice that the square of $\min_{W\in \mathcal{W}_{\alpha n}} \sigma_{\min}(\Phi(X)^\top W\Phi(X))$ corresponds to the sum of the  minimum $\alpha n$ squares, which is greater than $c_1n$ with high probability~\cite{boucheron2012concentration}. On the other hand, for arbitrary $\tilde{\theta}$, the additional error is at most $\epsilon \psi^{+}(\alpha n) = \mathcal{O}(\epsilon n)$. By using the uniform bound over all fixed $\theta$, and choosing $n\ge C d\log d$ for some large constant $c$, we can see that $\psi^{-}(\alpha n)$ is lower bounded by $\Omega(n)$ with high probability. 
	For getting the second term in (\ref{eqt:app-11}), we use the same idea as in the proof of Theorem \ref{thm:random}. 
	For any fixed $\theta_t$, the residuals for all the samples can be considered as generated from $m$ components, and can be reduced to a two-component setting. 
	Therefore, the numerator in $\kappa_t$ is again controlled by Theorem \ref{lem:help}.  Combining these results,  we have 
	$\kappa_t \le c  \left\{  \frac{\sqrt{\|\theta_t - \theta_{(j)}^\star \|_2^2 + \sigma^2}}{\min_{k\in[m]\backslash \{j\}}  \sqrt{\|\theta_t - \theta_{(k)}^\star\|_2^2 + \sigma^2} } \vee \frac{\log n}{n} \right\}$. 
	
	As a consequence,
	\begin{align*}
	\|\theta_{t+1} - \theta^\star \|_2 \le \kappa_t  \|\theta_t - \theta^\star \|_2 + c_1{\sqrt{ \kappa_t   }} \sigma + \frac{c_2\xi_t}{ n}\sigma,
	\end{align*} 
	where we require $n= \Omega ( d\log d)$. 
	Notice that for small $\ratiolearn$, in order to make $\kappa_t$ less than one, the noise should not be too large.  Otherwise, even if $\theta_t$ is very close to $\theta^\star$, because of the noise and the high density of bad samples, $|S_t\backslash S^\star|$ would still be quite large, and the update will not converge. 
	
\end{proof}

\begin{theorem} \label{thm:app-1}
	Following the setting in Lemma \ref{thm:nonlinear}, for the 
	given $\alpha < \alpha^\star$, $\Phi(X)$ being a regular feature matrix, and  $\alpha^\star > c_{\mathtt{th}}$, sample size $n=\Omega(d\log d)$, w.h.p., we have:
	\begin{align*}
		\|\theta^\star - \theta_{t+1} \|_2 \le \left( 1-c_1\eta (a^2 - \kappa_t b^2  \right) \| \theta_t - \theta^\star \|_2 + c_2 b \sqrt{\kappa_t  } \sigma  +  \frac{\eta b \xi_t}{ n} \sigma,
	\end{align*}
	where for $r$ being arbitrary output, $\kappa_t \le \frac{1}{2}$. For $r$ being random sub-Gaussian output, $\kappa_t \le c  \{ \frac{b}{a} \sqrt{\|\theta_t - \theta^\star\|_2^2 + \sigma^2} \vee \frac{\log n}{n}\}$.
\end{theorem}

\begin{theorem} \label{thm:app-2}
	Following the setting in Lemma \ref{thm:nonlinear}, 
	for the mixed regression setting in (2), suppose for some $j\in [m]$, $\alpha < \alpha_{(j)}^\star$. 
	Then, for  $n= \Omega( d\log d)$, w.h.p., the next iterate $\theta_{t+1}$ of the algorithm satisfies 
	\begin{align*}
	&\|\theta_{t+1}\! -\! \theta_{(j)}^\star \|_2 \le \left( 1-c_1\eta (a^2 - \kappa_t b^2  \right)  \|\theta_t - \theta_{(j)}^\star \|_2 +  c_1b{\sqrt{ \kappa_t   }} \sigma\! +\! \frac{c_2\eta b\xi_t}{ n}\sigma,
	\end{align*}
	where $\kappa_t \le c  \left\{  \frac{b \sqrt{\|\theta_t - \theta_{(j)}^\star \|_2^2 + \sigma^2}}{a \min_{k\in[m]\backslash \{j\}}  \sqrt{\|\theta_t - \theta_{(k)}^\star\|_2^2 + \sigma^2} } \vee \frac{\log n}{n} \right\}$.
\end{theorem}

The proof idea for the above two Theorems are similar to what we have shown in the proof of Theorem \ref{thm:random} and Theorem \ref{theorem:se}. 

\begin{theorem} \label{lem:help}
Suppose we have two Gaussian distributions $\mathcal{D}_1 = \mathcal{N}(0, \Delta^2), \mathcal{D}_2=\mathcal{N}(0, 1)$. We have $\ratiotrue n$ i.i.d. samples from $\mathcal{D}_1$ and $(1-\ratiotrue)n$ i.i.d. samples from $\mathcal{D}_2$. Denote the set of the top $\ratiolearn n$ samples with smallest abstract values as $S_{\ratiolearn n}$, where $\ratiolearn < \ratiotrue$. Then, with high probability, for $\Delta\le 1$,  
at most $ \left(c \max\left\{\Delta \left(1-\ratiotrue\right)  n, \log n\right\} \right)$ samples in $S_{\ratiolearn n}$ are from $\mathcal{D}_2$.
\end{theorem}

\begin{proof}

\textbf{Step I.} 
Let $S_1^\star, S_2^\star$ be the set of samples from $\mathcal{D}_1$, $\mathcal{D}_2$, respectively, and let $S_1 := S_{\alpha^\star n}$. 
Consider $|S_1\cap S_2^\star|$, 
by definition, let $\delta$ be the threshold between samples in $S_1\cap S_1^\star$ and samples in $S_1^\star \backslash S_1$ . Since there are at least $(1-c_{\alpha})\alpha^\star n$ samples in $S_1^\star$ that are not in $S_1$, 
by the sub-Gamma property of order statistics of Gaussian random variables~\cite{boucheron2012concentration}, we know 
\begin{align}\label{eqt:affine}
\Pr\left[ \delta > F_{\Delta}^{-1}(c_{\alpha}) + c_0\Delta \right] \le e^{-c_1\alpha^\star n},
\end{align}
where $F_\Delta{}$ is the distribution of the abstract value of random variable from $\mathcal{D}_1$. 
As a result, $\delta\le c_2\Delta$  with high probability. 

\textbf{Step II.}
On the other hand, for a random variable $u_2\sim \mathcal{D}_2$, we know that $\Pr[|u_2|\le \delta] \le \sqrt{\frac{2}{\pi}} \delta$, which is tight for small $\delta$. 
Let $\mathcal{M}_{\delta, i}$ be the event \textit{sample $u_i $ from $\mathcal{D}_2$  has abstract value less than $\delta$}, and a Bernoulli random variable $m_{i,\delta}$ that is the indicator of event $\mathcal{M}_{\delta, i}$ holds or not. 
Then, 
\begin{align*}
\mathbb{E} \left[ \sum_{i=1}^{(1-\ratiotrue) n} m_{i,\delta} \right] \le \sqrt{\frac{2}{\pi}}{\delta (1-\ratiotrue) n}.
\end{align*}
For independent Bernoulli random variable $x_i$s, $i\in [\tilde{n}]$ with $X=\sum_i x_i$ and $\mu = \mathbb{E}[X]$, Chernoff's inequality gives~\cite{vershynin2010introduction}
\begin{align*}
\Pr\left[ X \ge t \right] \le e^{-t}
\end{align*}
for any $t \ge e^2 \tilde{n} \mu$. 
In the above setting we consider, we have with high probability $1-n^{-c}$, $\sum_{i=1}^{(1-\ratiotrue) n} m_{i,\delta} \le c  \max\{ (1-\ratiotrue) n \Delta, \log n \} $. 

\end{proof}

\begin{figure*}[ht]
	
	\hfill 
	\begin{subfigure}{0.49\columnwidth}
		\centering
		\includegraphics[width=\linewidth]{figs/inconsistency-random-01-v4.pdf}
		\caption*{(a)}
	\end{subfigure}
	\hfill 
	\begin{subfigure}{0.49\columnwidth}
		\centering
		\includegraphics[width=\linewidth]{figs/inconsistency-random-1-v4.pdf}
		\caption*{(b)}
	\end{subfigure}

	\begin{subfigure}{0.49\columnwidth}
		\centering
		\includegraphics[width=\linewidth]{figs/ratio-random-v2.pdf}
		\caption*{(c)}
	\end{subfigure}
	\hfill 
	\begin{subfigure}{0.49\columnwidth}
		\centering
		\includegraphics[width=\linewidth]{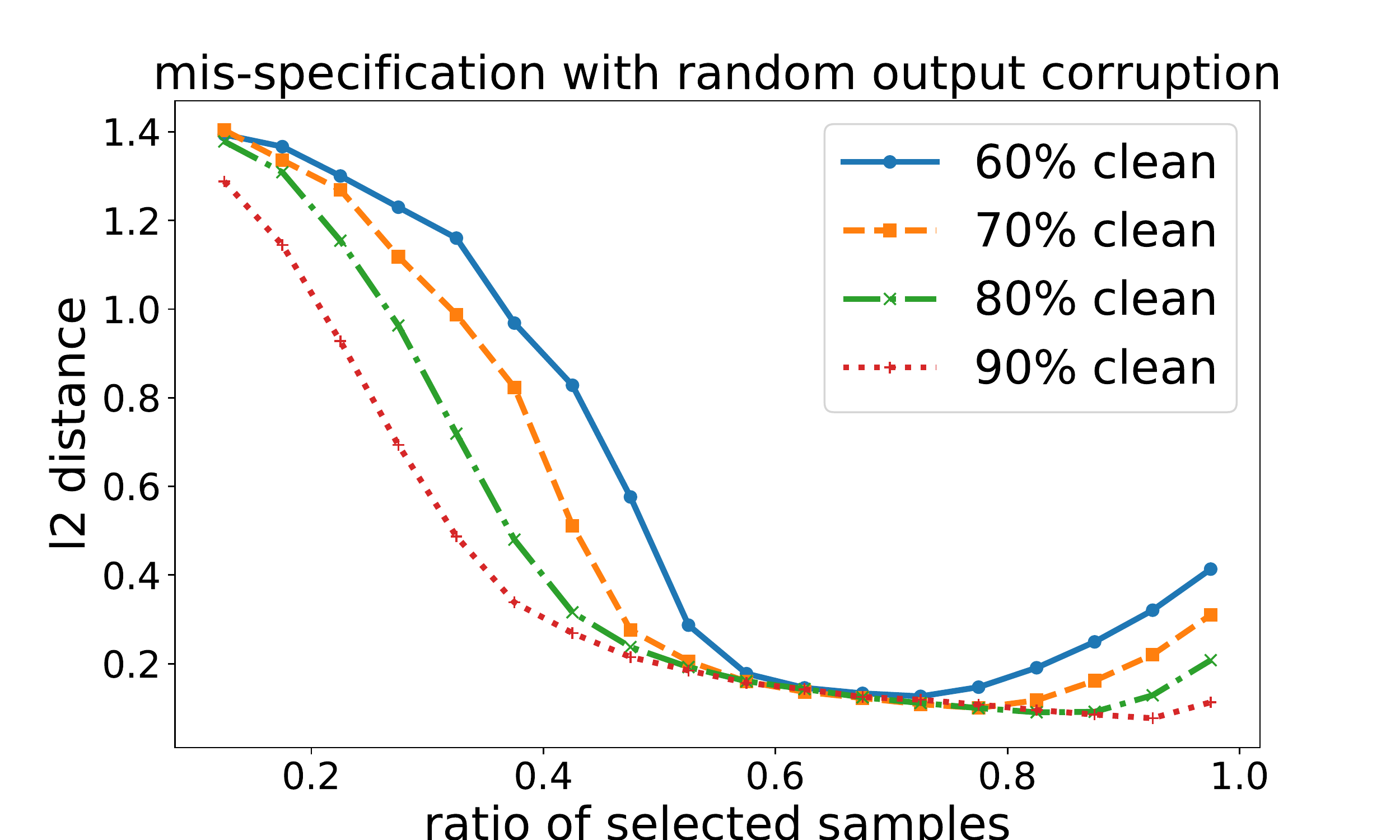}
		\caption*{(d)}
	\end{subfigure} 
	
	\begin{subfigure}{0.49\columnwidth}
		\centering
		\includegraphics[width=\linewidth]{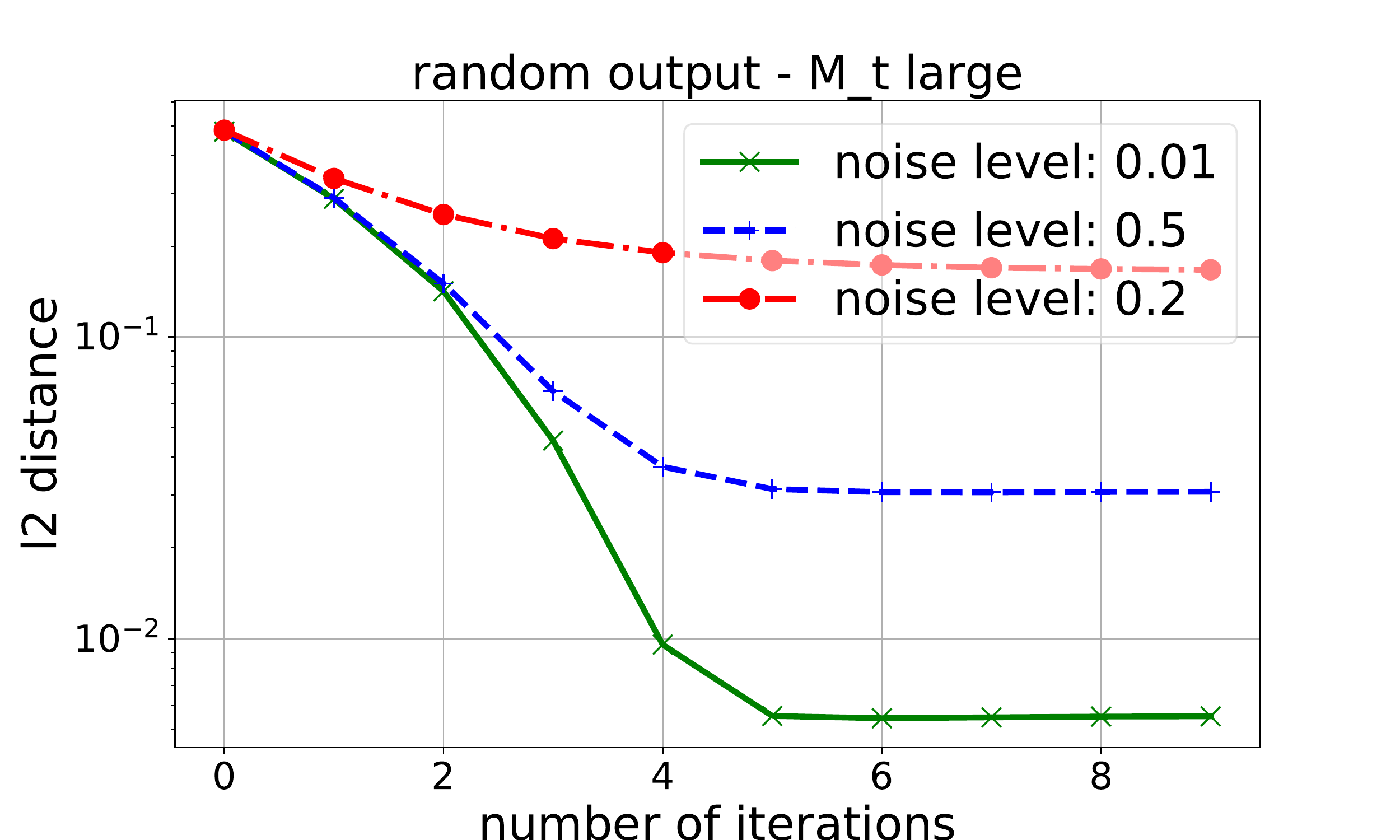}
		\caption*{(e)}
	\end{subfigure}
	\hfill 
	\begin{subfigure}{0.49\columnwidth}
		\centering
		\includegraphics[width=\linewidth]{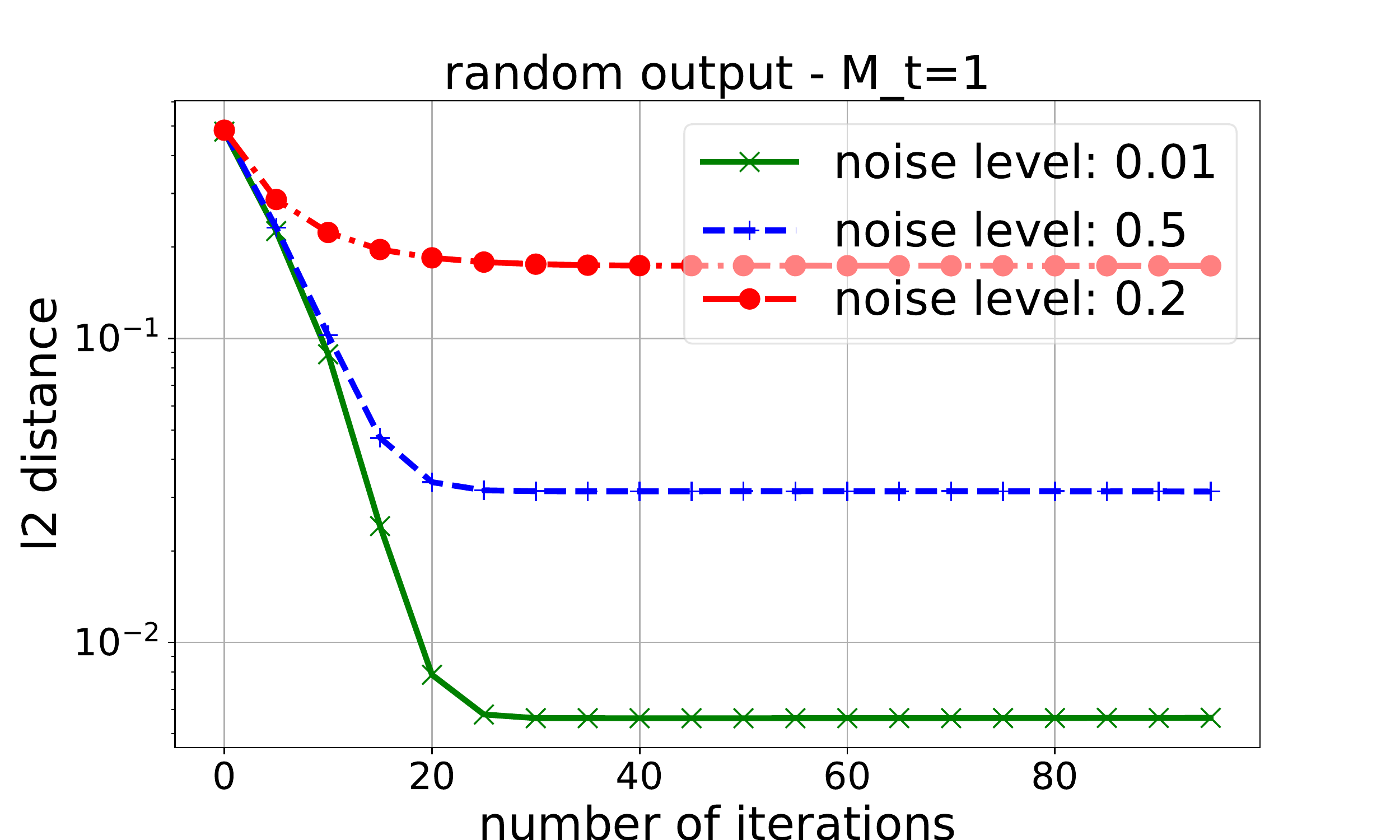}
		\caption*{(f)}
	\end{subfigure}	
	
	\caption{Synthetic experiments with random output: \textbf{(a): }  asymptotic performance under small measurement noise; \textbf{(b): } asymptotic performance under large measurement noise; \textbf{(c): }  performance under different good sample ratio; \textbf{(d): } the effect of mis-specification; 
	\textbf{(e): }  convergence rate of \algname with large $M_t$ (noise from $0.01$ to $0.2$ ; \textbf{(f): }  convergence rate of \algname with small $M_t$ (noise from $0.01$ to $0.2$ ).}
	\label{fig:sim-random}
\end{figure*}

\begin{figure*}[t]
	\begin{subfigure}{0.49\columnwidth}
		\centering
		\includegraphics[width=\linewidth]{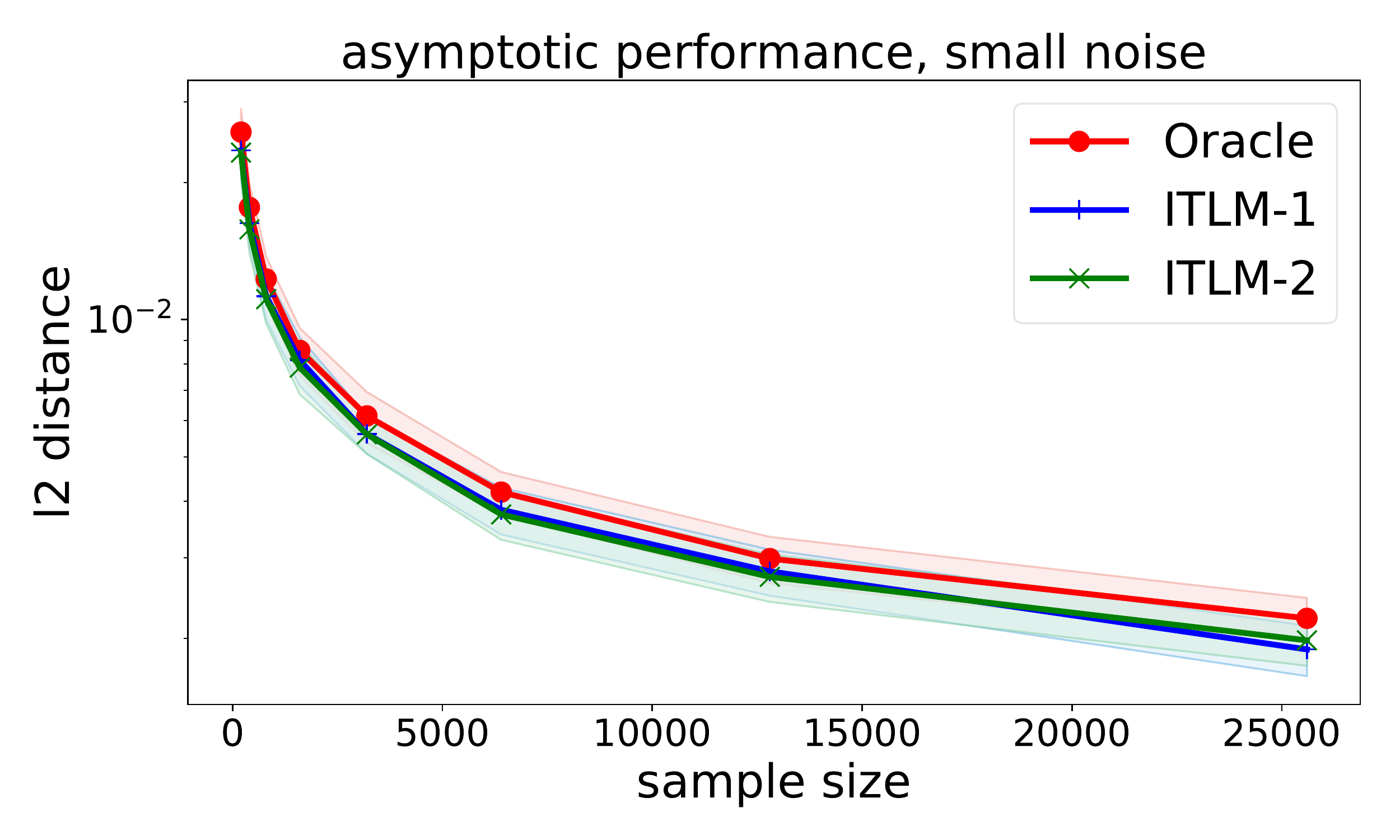}
		\caption*{(a)}
	\end{subfigure}
	\hfill 
	\begin{subfigure}{0.49\columnwidth}
		\centering
		\includegraphics[width=\linewidth]{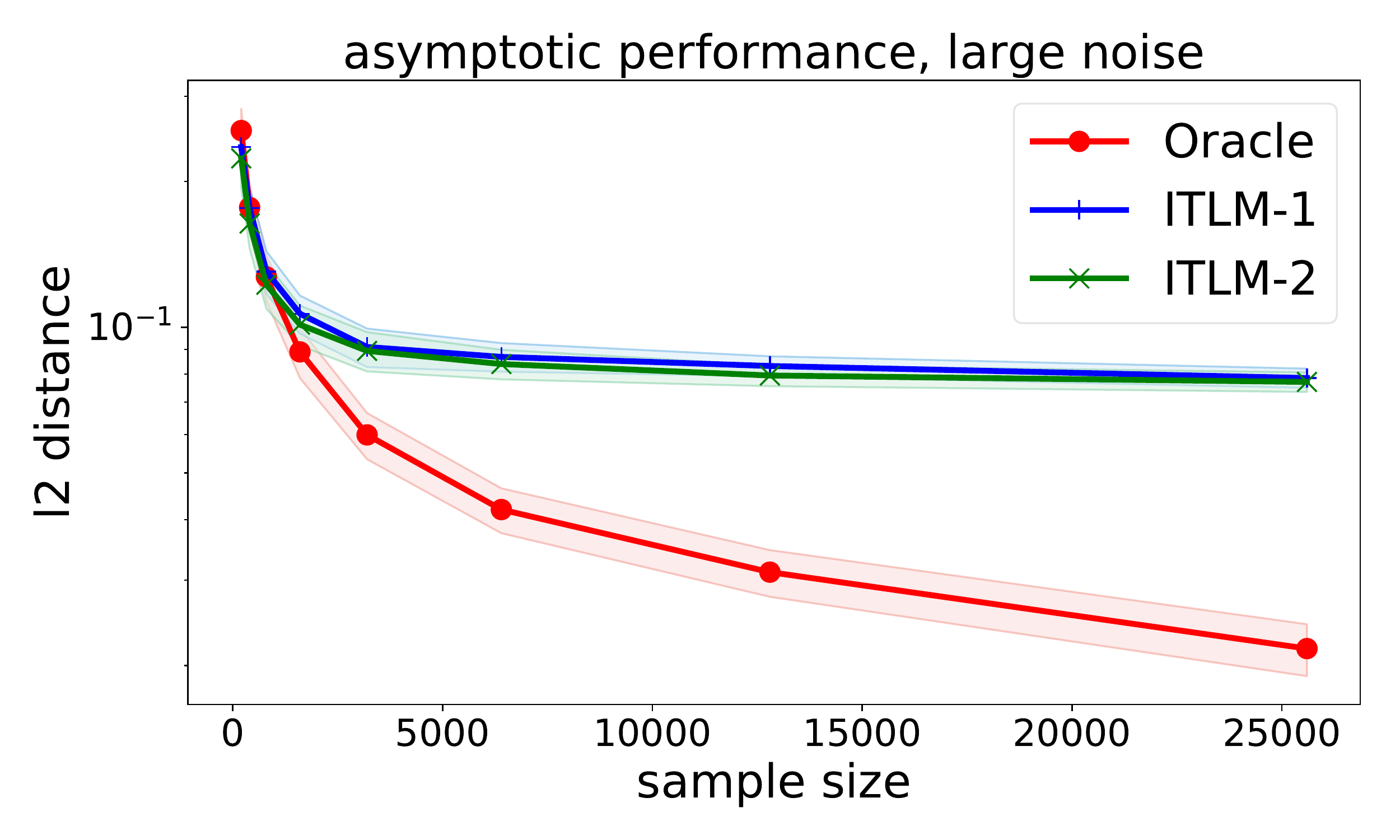}
		\caption*{(b)}
	\end{subfigure}
		
	\begin{subfigure}{0.49\columnwidth}
		\centering
		\includegraphics[width=\linewidth]{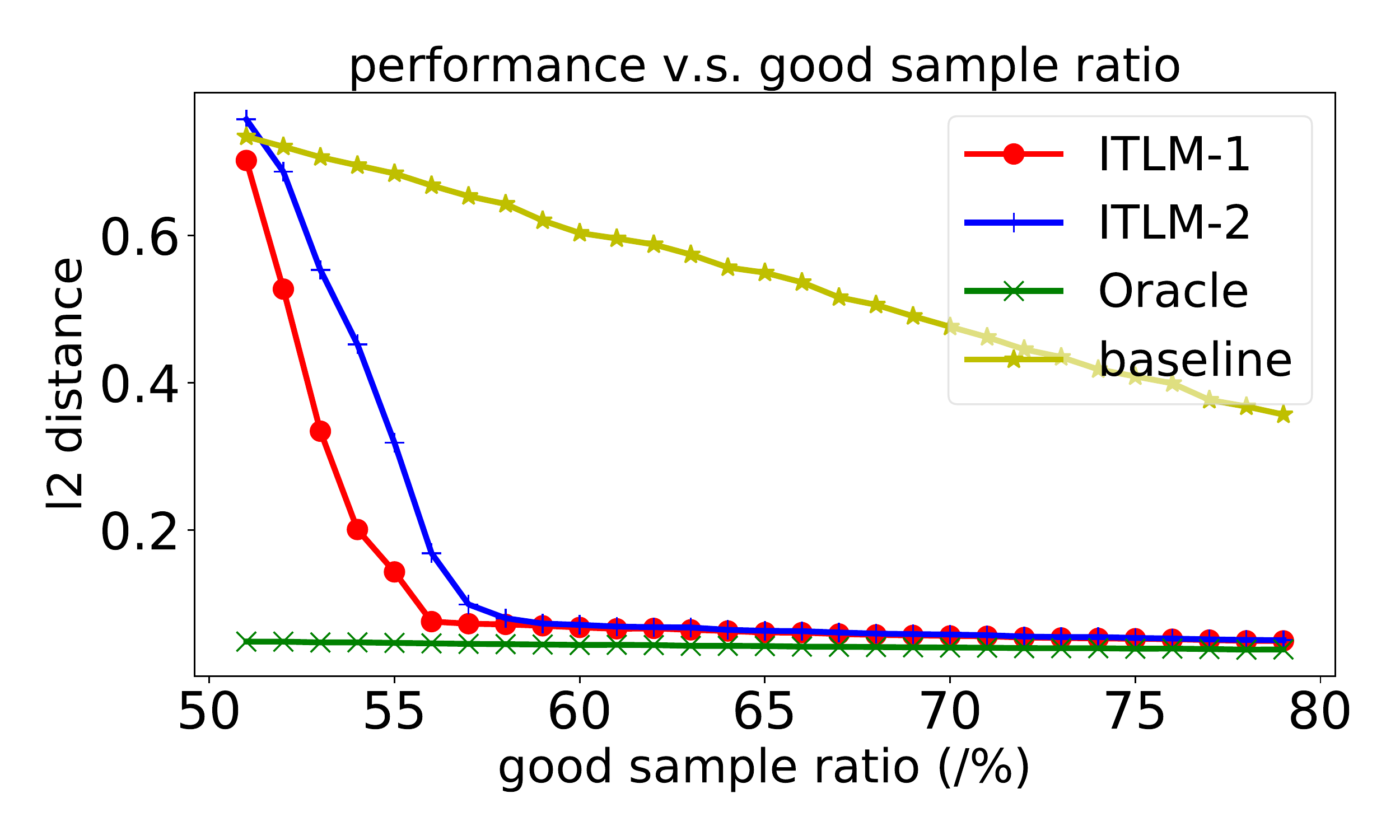}
		\caption*{(c)}
	\end{subfigure}
	\hfill 
	\begin{subfigure}{0.49\columnwidth}
		\centering
		\includegraphics[width=\linewidth]{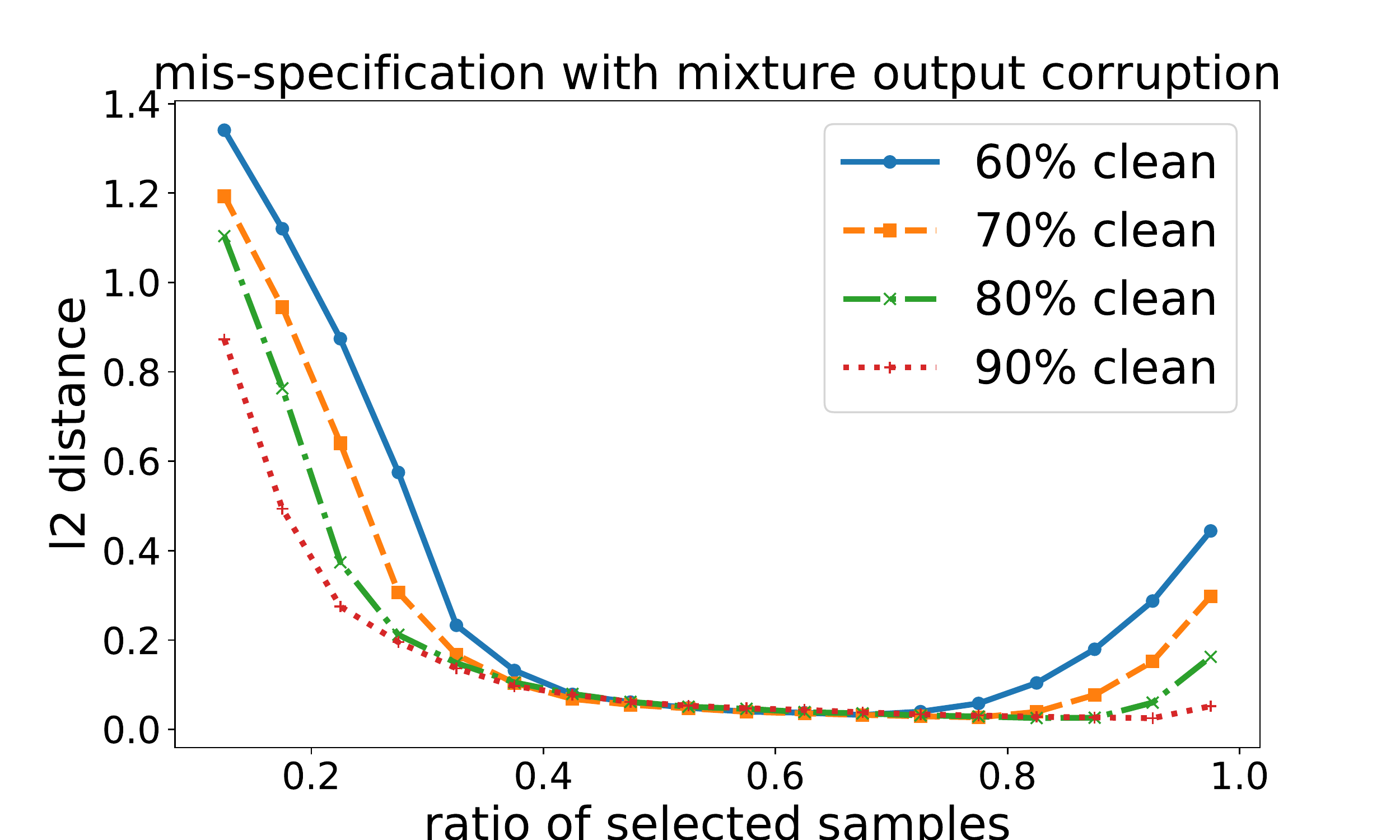}
		\caption*{(d)}
	\end{subfigure}
	
	\begin{subfigure}{0.49\columnwidth}
		\centering
		\includegraphics[width=\linewidth]{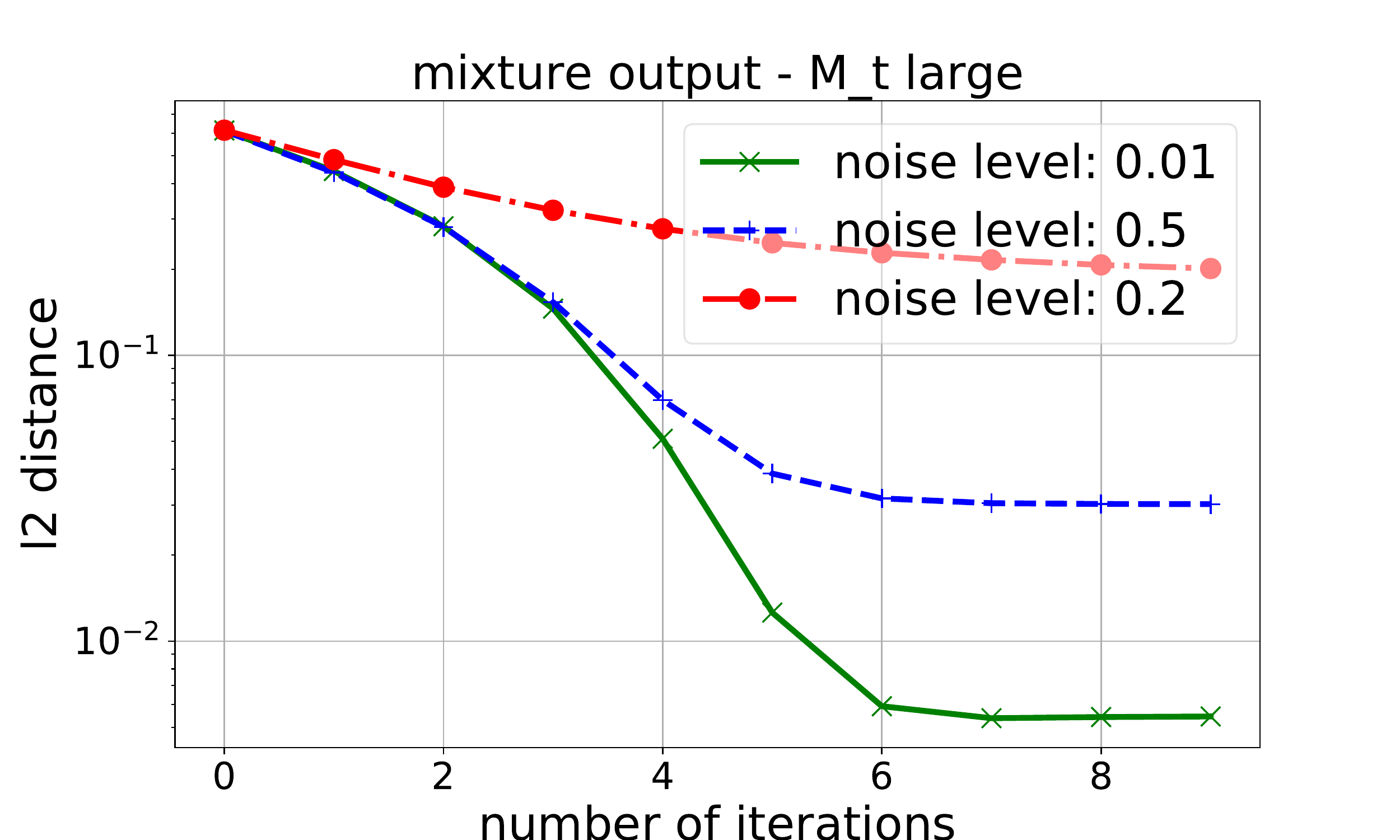}
		\caption*{(e)}
	\end{subfigure}
	\hfill 
	\begin{subfigure}{0.49\columnwidth}
		\centering
		\includegraphics[width=\linewidth]{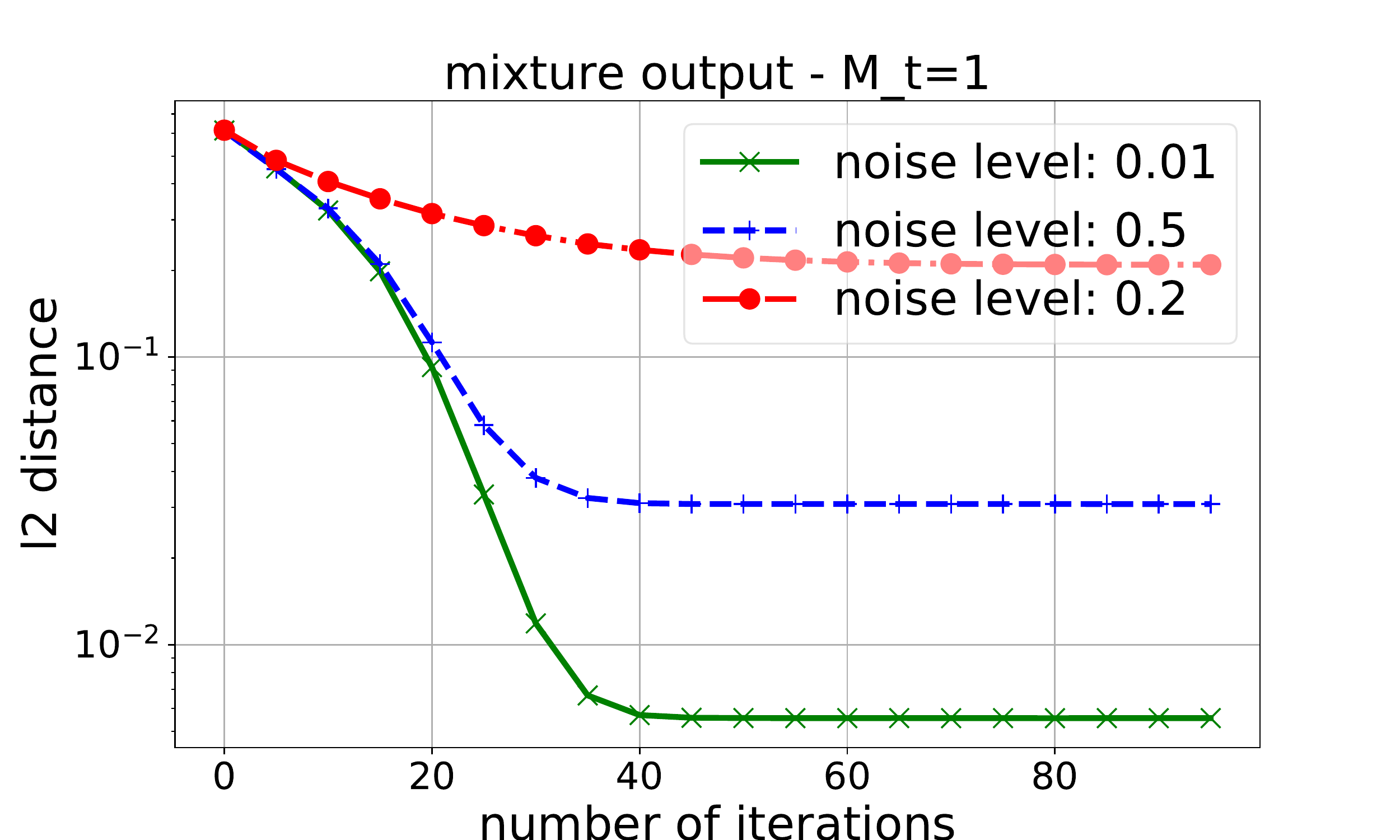}
		\caption*{(f)}
	\end{subfigure}	
	
	\caption{Synthetic experiments with mixture  output: \textbf{(a): }  asymptotic performance under small measurement noise; \textbf{(b): } asymptotic performance under large measurement noise; \textbf{(c): }  performance under different good sample ratio; \textbf{(d): } the effect of mis-specification; 
		\textbf{(e): }  convergence rate of \algname with large $M_t$ (noise from $0.01$ to $0.2$ ; \textbf{(f): }  convergence rate of \algname with small $M_t$ (noise from $0.01$ to $0.2$ ).}
	\label{fig:sim-mixture}
\end{figure*}

\begin{figure*}[t]
	\hfill 
	\begin{subfigure}{0.49\columnwidth}
		\centering
		\includegraphics[width=\linewidth]{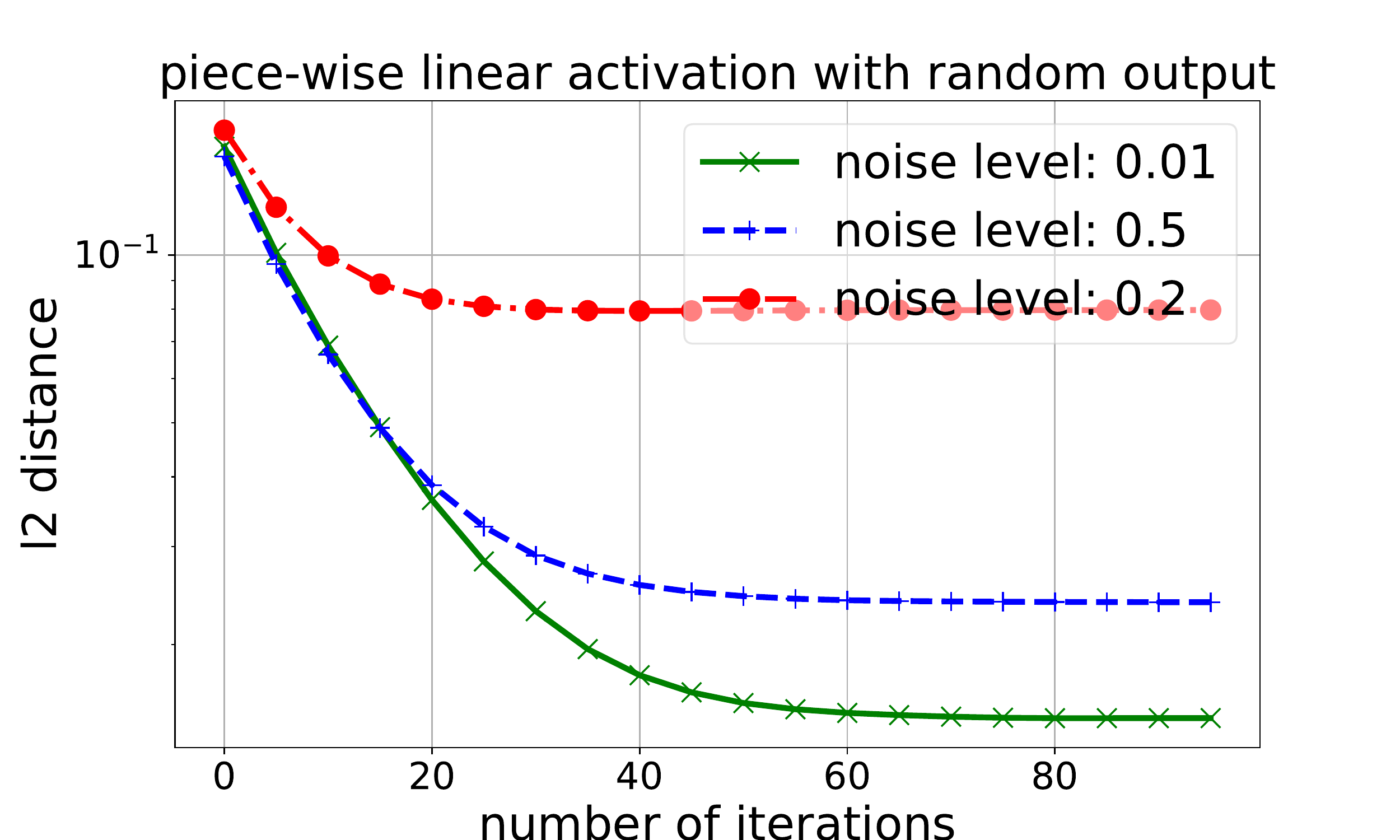}
		\caption*{(a)}
	\end{subfigure}
	\hfill 
	\begin{subfigure}{0.49\columnwidth}
		\centering
		\includegraphics[width=\linewidth]{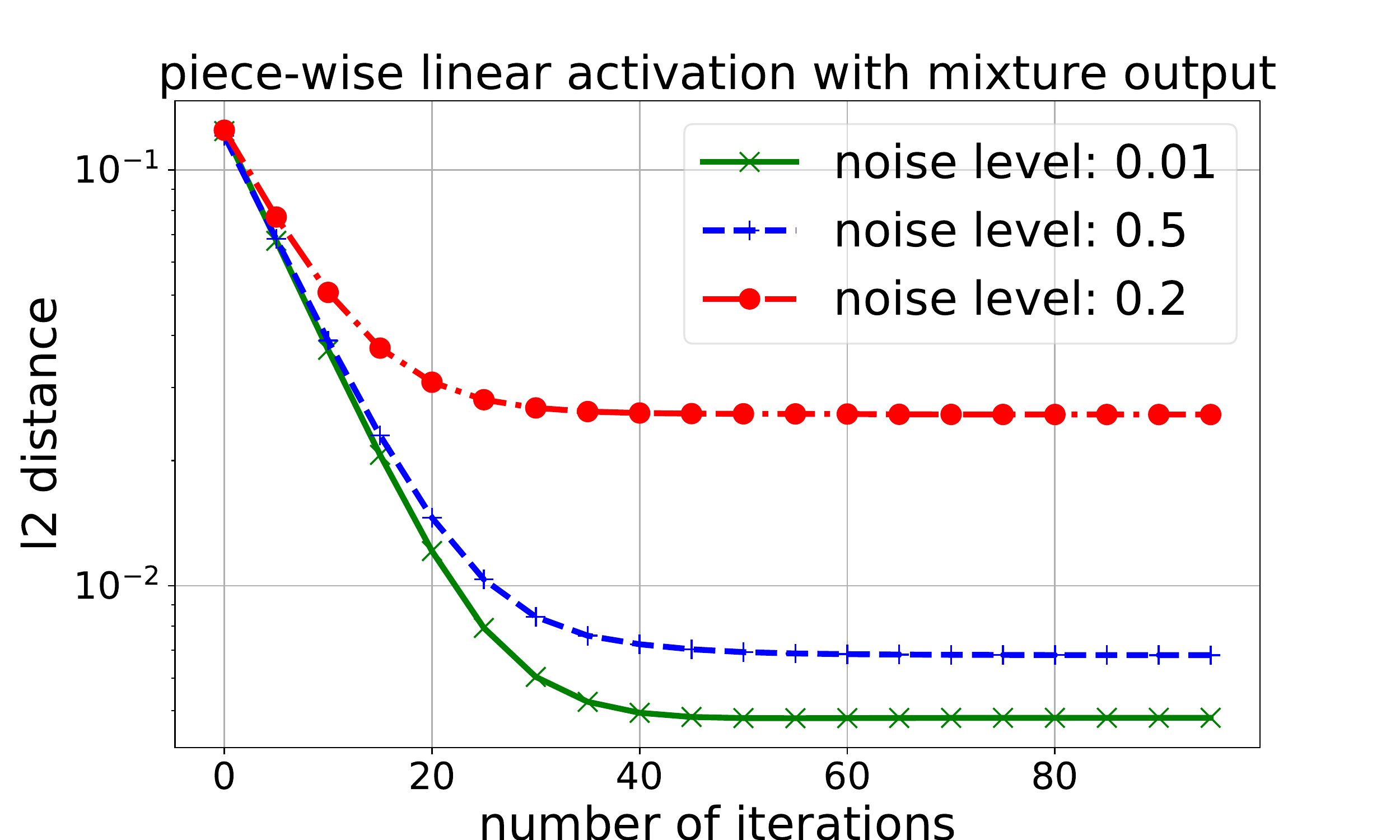}
		\caption*{(b)}
	\end{subfigure}
	
	\caption{Synthetic experiments with non-linear activation function: \textbf{(a): } $\|\theta^t - \theta^\star \|_2$ v.s. $t$ for random output setting; \textbf{(b): } $\|\theta^t - \theta^\star \|_2$ v.s. $t$  for  mixture output setting.}
	\label{fig:sim-nonlinear}
\end{figure*}

\section{Additional Synthetic Experiments}

In this section, we present the full results for the synthetic experiments, which aligns with our theoretic results in Section 5 (in the main paper). 
We focus on discussing behaviors for the linear case first, and then  provide results on the non-linear setting. 

\paragraph{Synthetic experiments for random output setting}
We generate the data according to (1), with $w(x) = x$, where we choose $\theta^\star$ to be a random unit vector with dimension $d=100$, every feature vector $\phi(x_i)$ is generated i.i.d. as a $d$-dimension normal spherical Gaussian. Random output $r_i$ is generated i.i.d. following $\mathcal{N}(0,1)$, which makes the distribution of both the bad and good outputs the same. We generate in total $n=1000$ samples, where $\alpha^\star$-fraction of them are clean samples and the rest are bad samples (with random output). The noise vector $e$ is generated i.i.d. Gaussian with variance $\sigma^2$. 

\paragraph{Synthetic experiments for mixed regression setting}
We generate the data following (2) with $w(x) = x$, for the settng of two components. The rest of the settings are similar to the random output setting, except for the bad samples, we select another $\theta_{(1)}$ with unit norm, orthogonal to $\theta^\star$. 

In Figure \ref{fig:sim-random} and Figure \ref{fig:sim-mixture}, we study:
\begin{itemize}
	\item (\textbf{Inconsistency}) The recovery performance as sample size increases, in both small-noise  and large-noise settings; 
	\item (\textbf{Recovery}) The recovery performance under different good sample ratios;
	\item (\textbf{Mis-specification}) The effect of mis-specified $\alpha$; 
	\item (\textbf{Convergence}) The convergence speed  under different noise levels, for both large and small $M_t$ settings.
\end{itemize}
All $y$-axis measures the $\ell_2$ distance, i.e., $\|\theta_t - \theta^\star\|_2$. 
Each data point  in the plots is based on $100$ runs of the same experiment to cancel out the random factors. 

\paragraph{Inconsistency} Figure \ref{fig:sim-random}-(a) \& (b) and Figure \ref{fig:sim-mixture}-(a) \& (b) show the result for asymptotic behavior. \algname-1 corresponds to our algorithm with large $M_t$, which corresponds to our analysis using the closed form solution at each update round. \algname-2 corresponds to our algorithm with $M_t=1$. The performance in both settings are quite similar: in the (b) plots with noise level $\sigma=1$,  as sample size increases, the oracle performance is getting better, while the performance of \algname does not keep improving, which shows the inconsistency of the algorithm. However, in the (a) plots with small noise ($\sigma = 0.1$), the difference between oracle and \algname is not significant, for sample size less than 25k. However, as sample size keeps getter larger, we will observe the behavior of inconsistency for \algname. The observation matches with our results in Theorem 7 \& 8, where our per-round convergence property will guarantee the recovered parameter is within a noise ball to the ground truth parameter. 

\paragraph{Recovery} Figure \ref{fig:sim-random}-(c) and Figure \ref{fig:sim-mixture}-(c) show the recovery performance when good sample ratio varies. \algname-1 and \algname-2 perform similarly. As good sample ratio gets larger, the algorithm is capable of recovering close to the ground truth with high probability. Here, noise level $\sigma=0.2$, $\alpha$ is set as $\alpha^\star -5\%$ by default. 

\paragraph{Mis-specification} In Figure \ref{fig:sim-random}-(d) and Figure \ref{fig:sim-mixture}-(d), we study the recovery behavior for different mis-specified $\alpha$s. We see that the recovery performance is not very sensitive to the selection of $\alpha$, especially when the dataset has more clean samples. 

\paragraph{Convergence} In Figure \ref{fig:sim-random}-(e) \& (f), and Figure \ref{fig:sim-mixture}-(e) \& (f), we see the convergence is more than linear before the learned parameter gets into the noise-level close to the ground truth, for both settings. 
This convergence behavior, for both the random output and mixture output settings, matches with our results in Theorem 7 and Theorem 8. 

\paragraph{Non-linear activation functions}
In Figure \ref{fig:sim-nonlinear}, we present convergence result for a non-linear setting: we choose $w()$ to be a piece-wise linear function, i.e., $w(x) = x$ if $x<0$, and $w(x) = 1.2x$ if $x\ge 0$. We keep all other settings exactly the same as in previous synthetic experiments. We see that the \algname has similar convergence behavior as in the linear setting.

\clearpage
\newpage 

\section{Additional Experiments and Implementation Details}

All experiments are implemented using MXNet and gluon. Here, we add more experimental details and supporting experimental results. 

\subsection{Details for the image classification task with random/systematic label errors}

\paragraph{Training details:} We use batch size $1000$ with learning rate $0.3$ for subsampled MNIST dataset, and batch size $256$ with learning rate $0.1$ for CIFAR-10 dataset, with naive sgd as the optimizer. We use $80$ epochs for naive training, and decrease step size at the $50$ epoch by $5$. The results for MNIST dataset is reported as the median of $5$ random runs. 
In all the experiments, there is \textbf{no clean sample} in both the training set and the validation set. The reported accuracy is based tested on the true validation set, but the algorithm saves the best model based on the accuracy on the \textbf{bad} validation set, which has the same corruption pattern as the training set. 

\subsection{Additional experiments for image generation}

\paragraph{Training details:} We use the popular DC-GAN architecture, and the loss for training is re-written in (\ref{eqt:gan-loss}), which is also used for the update step in \algname. 

\begin{equation} 
	L_S^{\mathtt{GAN}}(\theta^D,\theta^G) := \frac{1}{|S|}\sum_{i\in S} \log D_{\theta^D}(s_i) +  \mathbb{E}_{z\sim p_\mathcal{Z}(z)}\left[ \log(1-D_{\theta^D}(G_{\theta^G}(z))) \right] 
	\label{eqt:gan-loss}
\end{equation}
\begin{align}
	S_t \leftarrow \arg\min_{S : |S|=\alpha n} \sum_{i \in S} D_{\theta^D_t}(s_i) \label{eqt:gan-step1}
\end{align} 

\begin{figure}
	\begin{subfigure}{0.23\linewidth}
		\includegraphics[width=\linewidth]{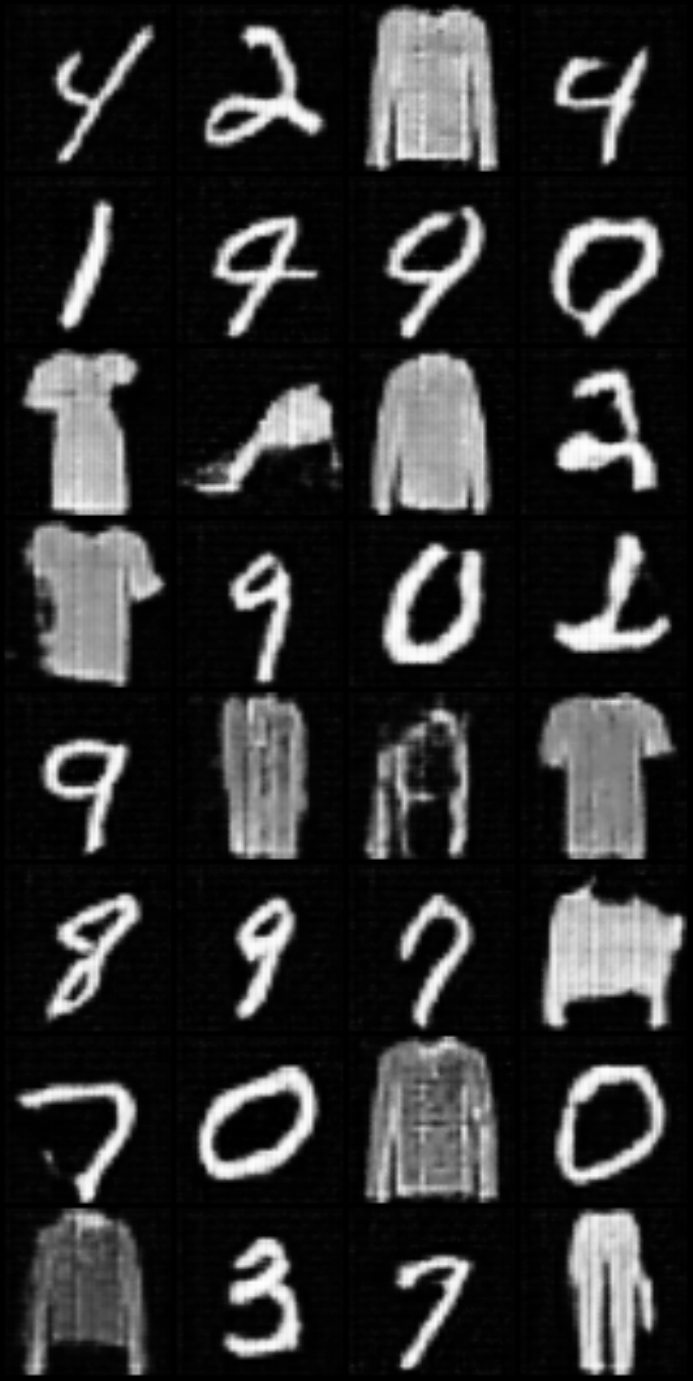}
		\caption{baseline} \label{fig:mnist-fashion-failure-a}
	\end{subfigure}
	\hspace*{\fill}
	\begin{subfigure}{0.23\linewidth}
		\includegraphics[width=\linewidth]{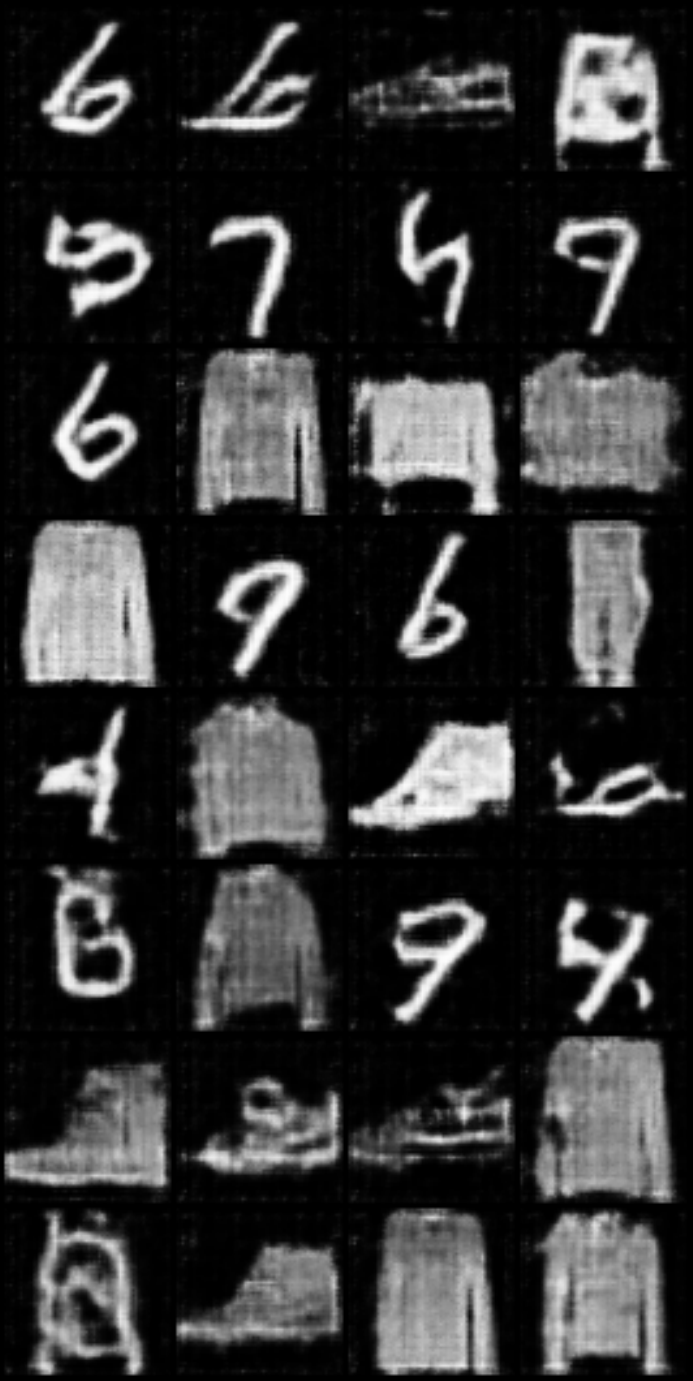}
		\caption{1st iteration} \label{fig:mnist-fashion-failure-b}
	\end{subfigure}
	\hspace*{\fill}
	\begin{subfigure}{0.23\linewidth}
		\includegraphics[width=\linewidth]{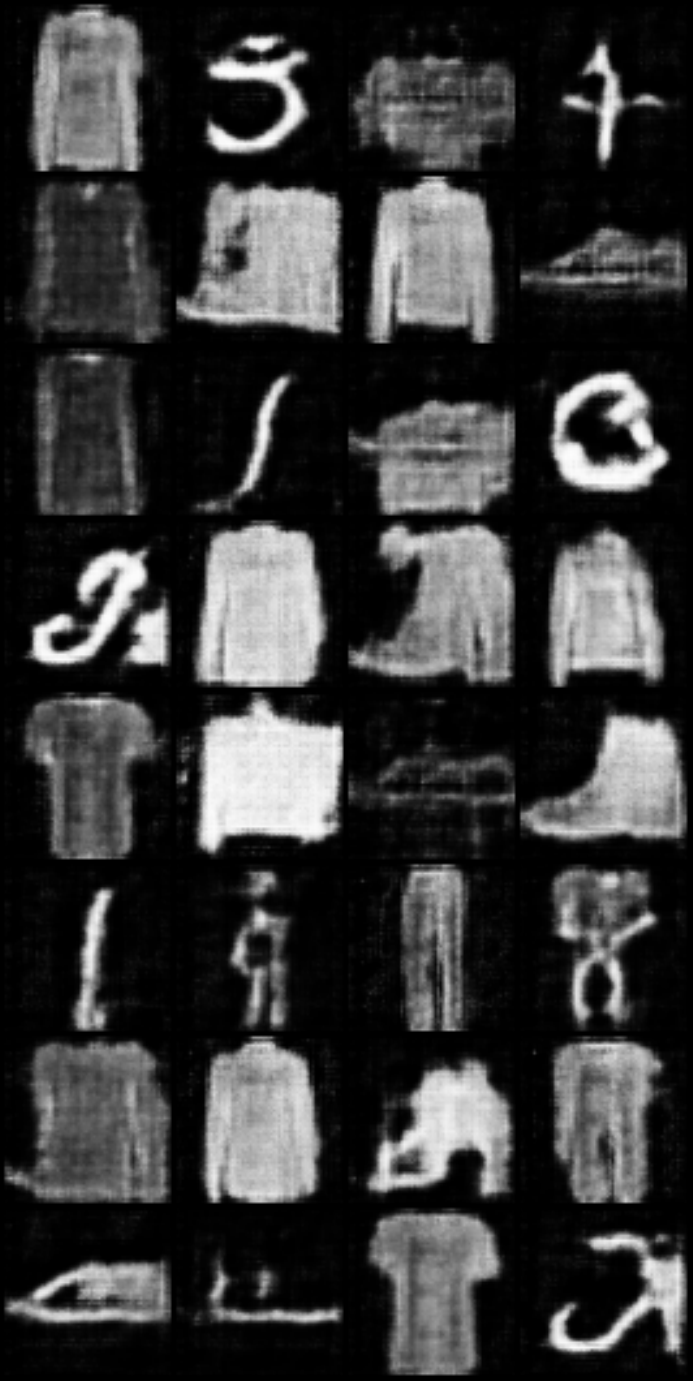}
		\caption{3rd iteration} \label{fig:mnist-fashion-failure-c}
	\end{subfigure}
	\hspace*{\fill}
	\begin{subfigure}{0.23\linewidth}
		\includegraphics[width=\linewidth]{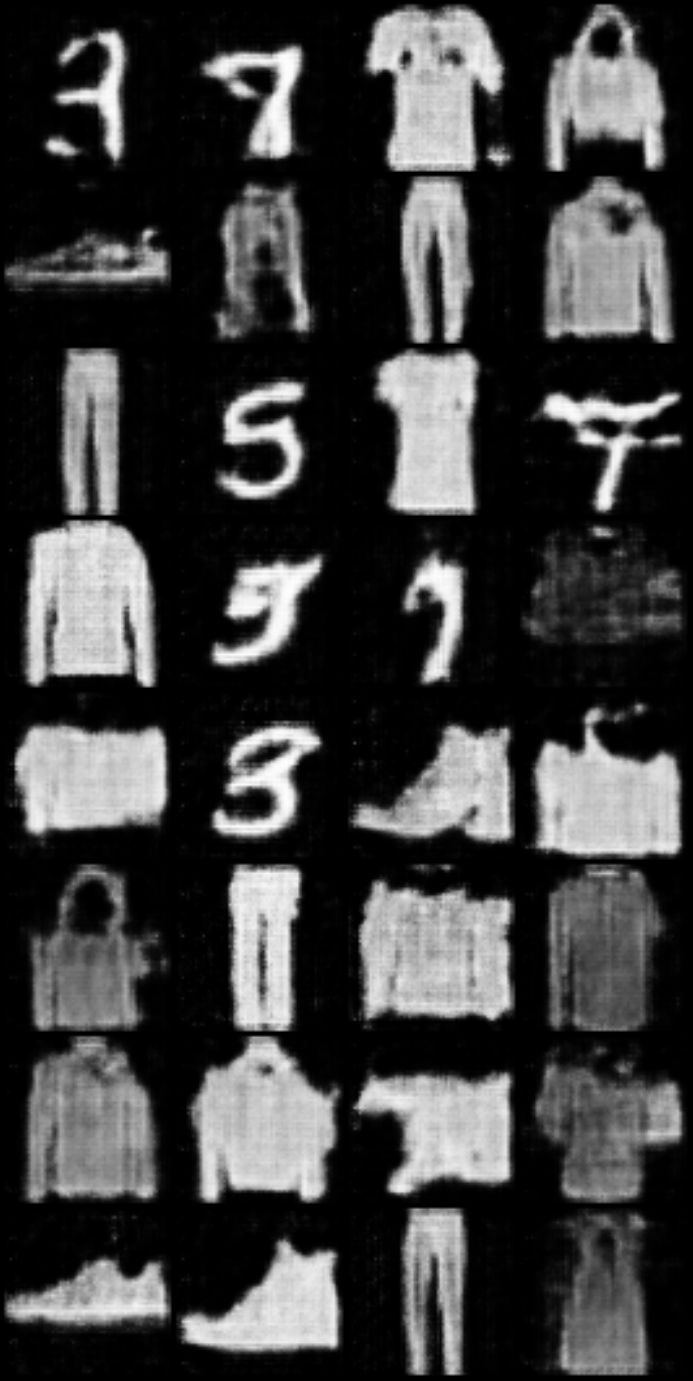}
		\caption{5th iteration} \label{fig:mnist-fashion-failure-d}
	\end{subfigure}
	\hspace*{\fill}
	\caption{{\bf Illustrative failure case:} This figure shows that when the fraction of bad samples is too large, ILFB cannot clean them out. The setting is exactly the same as in Figure 3 (in the main paper), but now with 60\% MNIST clean images + 40\% Fashion-MNIST bad images. We can see that now the $5^{th}$ iteration still retains the fake fashion images.}
	\label{fig:mnist-fashion-failure}
\end{figure}

\begin{table}[!ht]
	\centering
	\caption{MNIST GAN: comparison with other choices}
	\footnotesize
	\begin{tabular}{lccccccc}
		\toprule 
		dataset & \multicolumn{7}{c}{MNIST}  \\
		\midrule
		$\alpha^\star=\frac{\mbox{\# clean}}{\mbox{\# total}}$  & \textbf{Baseline} & \textbf{\algname} &  \textbf{Centroid} & \textbf{1-Step}  & \textbf{$\Delta\tau = 10\%$} & \textbf{$\Delta\tau = 15\%$} & \textbf{$\Delta\tau = 20\%$} \\
		\midrule
		70\% & 70 & 97.00 & 61.46 &  77.77 & 83.33 & 78.06 & 83.59  \\
		80\%  & 80 & 100.00 & 77.46  &  76.84&  98.80 & 99.56 & 97.77  \\
		90\% & 90 & 100.00 & 89.57  &  91.90 & 98.85 & 99.01 & 98.04   \\
		\bottomrule
	\end{tabular}
	\label{app:table-gan}
\end{table}

\paragraph{Experimental settings:} 
In this part, we present additional experimental results, in order to verify the performance of \algname under different parameter settings, and compare with other algorithms. More specifically, we present the results using the following methods/algorithms:

\begin{itemize}
	\item \textbf{Baseline}: naive trainig using all the samples; 
	\item \textbf{\algname}: our proposed iterative learning algorithm with $5$ iterations, using a mis-specified $\tau$ which is $5\%$ less than the true value; 
	\item \textbf{Centroid}: using the centroid of the input data to filter out outliers. For classification task, we calculate the centroids for the samples with the same label/class and filter each class separately; 
	\item \textbf{1-Step}: \textbf{\algname} algorithm with a single iteration;
	\item \textbf{$\Delta \tau = \tau^\star - \tau \in \{10\%, 15\%, 20\%\}$}: \textbf{\algname} under different mis-specified $\tau$ value,
\end{itemize}
under MNIST generation with Fashion-MNIST images.

For the generation task (Table \ref{app:table-gan}), we present the ratio of true MNIST samples selected by each method. For the baseline method, since the DC-GAN is trained using all samples, the reported value is exactly the $\tau^\star$. 

\paragraph{Results:}
Table \ref{app:table-gan} shows the performance of generation quality under different noise levels. We observe that centroid method does not work, which may due to the fact that all MNIST and Fashion-MNIST images are hard to be distinguished as two clusters in the pixel space. Notice that there are in fact $20$ clusters ($10$ from MNIST, and $10$ from Fashion-MNIST), and we are interested in $10$ of them. \algname works well since it automatically learns a clustering rule when generating on the noisy dataset. For example, for $\tau^\star = 80\%$, even with a mis-specified $\tau = 60\%$, \algname is capable of ignoring almost all bad samples. Again, we also observe significant improvement of \algname over its $1$-step counterpart. 

We also have results showing that \algname  works well for generation when the corrupted samples are pure Gaussian noise. However, we do not think it is a practical assumption, and the result is not presented here. 

In Figure \ref{fig:mnist-fashion-failure}, we present a result under large bad sample ratio: 60\% clean MNIST images with 40\% bad Fashion-MNIST images. The algorithm, after the 5-th iteration, tries to filter out  all digit-type images.

\end{document}